\documentclass{article}
\usepackage{graphicx} 
\usepackage{wrapfig}
\usepackage[utf8]{inputenc}
\usepackage{url}
\usepackage{hyperref}
\usepackage{algorithm}
\usepackage{algpseudocode}
\usepackage{bbding}
\usepackage{pifont}
\usepackage{amsmath,amssymb,amsthm}
\theoremstyle{plain}
\newtheorem{theorem}{Theorem}
\usepackage{tabularx}
\theoremstyle{definition}

\usepackage{mathtools}
\usepackage{placeins}

\theoremstyle{remark}
\newtheorem{remark}{Remark}
\usepackage{multicol}
\usepackage[
left=14mm, right=14mm, top=22mm, bottom=32mm,
]{geometry}
\usepackage{cleveref}
\usepackage{todonotes}
\usepackage{ bbold }
\usepackage{booktabs}

\usepackage{mathtools}
\usepackage{tikz}

\usepackage{sansmath}
\usepackage[most]{tcolorbox}
\usetikzlibrary{arrows.meta, positioning, decorations.pathmorphing}

\newtheorem{definition}{Definition}
\newtheorem{lemma}{Lemma}
\newtheorem{corollary}{Corollary}
\newtheorem{claim}{Claim}
\usepackage{svg}
\usepackage{xfrac}
\usepackage{subcaption}
\usepackage[square,numbers]{natbib}
\author{Emily Cheng and Francesca Franzon}

\title{Principles of semantic and functional efficiency \\in grammatical patterning}

\date{\small{
orcid: \href{https://orcid.org/0000-0003-3209-688X}{0000-0003-3209-688X} - \href{https://orcid.org/0000-0003-0503-2792}{0000-0003-0503-2792} \\
\url{emilyshana.cheng@upf.edu} - \url{francesca.franzon@upf.edu}\\
\vspace{2mm}
\emph{Department of Translation and Language Sciences\\Universitat Pompeu Fabra\\Barcelona}
}
}

\begin{document}

\maketitle
\begin{abstract}
Grammatical features such as number and gender serve two central functions in human languages. While they encode salient semantic attributes like numerosity and animacy, they also offload sentence processing cost by predictably linking words together via grammatical agreement.
Grammars exhibit consistent organizational patterns across diverse languages, invariably rooted in a semantic foundation—a widely confirmed but still theoretically unexplained phenomenon. To explain the basis of universal grammatical patterns, we unify two fundamental properties of grammar, semantic encoding and agreement-based predictability, into a single information-theoretic objective
under cognitive constraints, accounting for variable communicative need. Our analyses reveal that grammatical organization provably inherits from perceptual attributes, and our measurements on a diverse language sample show that grammars prioritize functional goals, promoting efficient language processing over semantic encoding.\\
\linebreak
\end{abstract}
\vspace{6mm}

\begin{multicols}{2}
Languages can transmit any meaning through words and their combinations, yet remain feasible to store and process.
Several studies indicated how linguistic structure underlies efficiency in organizing meanings~\cite{Kemp_Regier_2012,Zaslavsky_Kemp_Regier_Tishby_2018} and sustaining processing \cite{futrell2015large,ferrer2020optimal,piantadosi2011word,blevins2017zipfian,caplan2020miller}.
Here, we show how these two aspects interact in shaping languages' grammars, as well as their use in single communicative instances.
Central to linguistic structure, grammars enable both meaning encoding and sentence parsing as essential functionalities \cite{greenberg1963some, corbett1991gender, corbett2000number, corbett2012features}. 
For example, the grammatical plural value in `these cats are sleeping' means that more than one cat is involved; at the same time, it traces the related words in the sentence, `these' and `are', which also occur in the plural due to agreement.

A key question is how grammatical organization reflects the joint need to support semantic encoding and sentence parsing, roles that have been investigated separately, but never in a unified framework.
We address this question building on the functionalist approach to linguistics~\cite{futrell2023information,tishby2000information}, modeling how language structure reflects communicative goals and constraints. We unify meaning transmission and processing goals within an information-theoretic model, showing that grammars are shaped to optimally trade off encoding benefits and costs, and formalizing how communicative need modulates this trade-off both in the overall grammatical organization and in the use of grammatical values. 
Our formalization explains cross-linguistic similarities observed across grammars and related word occurrence patterns, and provides a new interpretative framework to understand encoding optimization in language variation and change.

\subsubsection*{Grammar encodes meaning}
Across living and historical languages of all families, grammars consistently encode information about perceptually or culturally salient aspects of the referents~\cite{golston2018phi, strickland2017language, franzon2019non}. Number, a quasi universal, encodes \emph{semantic} attributes concerning the numerosity of referents, mostly differentiating singular and plural values~\cite{corbett2000number, franzon2019non}; gender encodes some constant properties of referents, such as animacy or sex, including values like animate - inanimate; masculine - feminine - neuter~~\cite{dixon1979ergativity, Silverstein1986, corbett1991gender}. 
Across languages, most grammatical systems comprise strikingly similar sets of values~\cite{greenberg1963some, corbett1991gender, corbett2000number, grambank_release, wals-30, wals-34}.\footnote{Here, we define a value not based on its form but on its role in grammatical agreement. The form that values can assume is variable. Even within languages, the same grammatical value can be marked in more or less consistent ways~\cite{pescuma2021formfunction, plag2024german}, for example German plurals show several different suffixes (i.e., -Ø, -e, -er, -n, -s) \cite{regel2019processing}. 
Also, some languages mostly mark gender and number in a single affix, e.g, in Italian gatt-i[Masc, Plur], 'cats'. Other languages mark them in distinct affixes, e.g. in Spanish gat-o[Masc]-s[Plur], 'cats') \cite{finkel2007principal}. 
Properties of form have been linked to effort minimization constraints in learning and processing~\cite{blevins2017zipfian, Mollica_Bacon_Zaslavsky_Xu_Regier_Kemp_2021, pescuma2021formfunction} but here we focus on the values (e.g., plural) independently of their form variations (the suffixes used for plural)} 

Previous research has explained these cross-linguistic similarities to result from optimal meaning encoding, framed within an information-theoretic approach to communication~\cite{Futrell_Hahn_2022,gibson2019efficiency,shannon1948mathematical}. 
The core assumption is that communication aims to transmit meanings unambiguously~\cite{Mollica_Bacon_Zaslavsky_Xu_Regier_Kemp_2021,Kemp_Regier_2012,Zaslavsky_Kemp_Regier_Tishby_2018}.

In successful communication, a sender encodes an intended meaning into a symbol, and a receiver decodes that symbol to recover the original meaning. An efficient set of symbols will achieve optimal \emph{compression}, namely comprise enough symbols to precisely distinguish the possible meanings a sender can speak about (that is, the benefit), while also minimizing the cognitive load of memorizing many symbols (that is, the cost). Different solutions to this tradeoff have explained the organization of existing sets of symbols, for example, color and kinship terms in the lexical domain~\cite{Futrell_Hahn_2022, gibson2019efficiency, Kemp_Regier_2012, Zaslavsky_Kemp_Regier_Tishby_2018}, as well as grammatical systems encoding semantic attributes such as tense, number, and evidentiality~\cite{Mollica_Bacon_Zaslavsky_Xu_Regier_Kemp_2021}.

These approaches focus on finding an optimal inventory of symbols given a set of meanings (e.g., color terms given color perception \citep{regier2007color}, or grammatical number values given numerical perception~\citep{Mollica_Bacon_Zaslavsky_Xu_Regier_Kemp_2021}), while not considering how such symbols are used in actual instances of communication. The underlying assumption is that the encoding is fixed and deterministic: once a meaning is mapped to a symbol, that symbol will consistently be used in every communicative instance, to maximize successful decoding by the receiver. This seems trivial in some lexical domains: a speaker would unlikely use the word `blue' to describe an instance of what is categorized as `red'. 
But are grammatical values used this way? 

\begin{figure*}[t]
\centering
    \includegraphics[width=\textwidth]{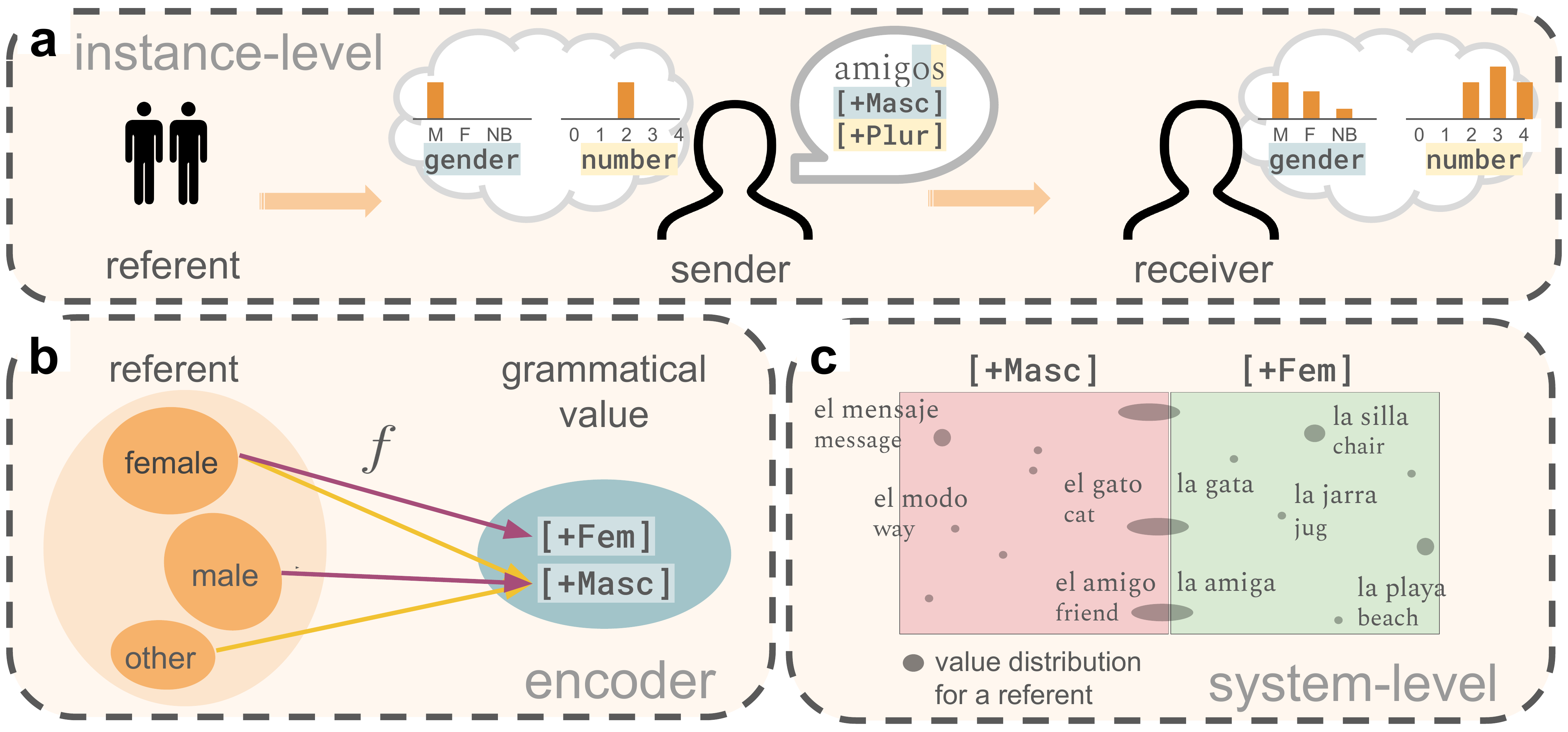}
    \caption{\small{\textbf{(a)} The word-level objective. A \emph{sender} (middle) observes an instance of referent (left) $x=$friend with semantic attributes $A_{\text{gender}}=$\emph{+male} and $A_{\text{number}}=$\emph{+two}. In Spanish, the sender utters the word `amigos' using the grammatical masculine [+Masc] and plural [+Plur]. Hearing the utterance, the receiver (right) attempts to reconstruct the original meaning $A_{\text{gender}}=$\emph{+male} and $A_{\text{number}}=$\emph{+two}, shown as a posterior distribution over the attributes $A$ given the uttered grammatical values. \textbf{(b)} A closer look at the sender. In the example, a referent such as \emph{friend} (left) takes on values in $\{$\emph{+female, +male, other}$\}$; each point in the space is an instance of \emph{friend}. The sender compresses the input distribution of referent attributes via encoder $f$ to, e.g., two grammatical values [+Fem] and [+Masc]. \textbf{(c)} An example of the system-level view for a language with masculine-feminine grammatical gender. Each ellipse denotes the distribution over grammatical values for a given referent over all of its occurrences (tokens). For instance, `el modo' (way, left edge) takes [+Masc] 100\% of the time, and `el gato' (cat, center) takes [+Masc] with probability $p$ and [+Fem] with probability $1-p$. Overall, probability mass over tokens is evenly distributed between [+Masc] and [+Fem], permitting optimal agreement-based discriminability. The distribution for each referent lies in one region or straddles both according to its individual word-level objective, taking either one value (`el modo') or multiple values (`el gato' [+Masc], `la gata' [+Fem]).}}
    \label{fig:keyfig}
\end{figure*}

The mapping between grammatical values and semantic attributes instead looks inconsistent in actual language use~\cite{bybee1988morphology, matthews1991morphology, haspelmath2017explaining, haspelmath2021explaining}. 
A typical example comes from Romance languages, like Spanish, where grammatical gender values (masculine vs.~feminine, from now on [Masc] vs.~[Fem]) mark sex-related attributes in some animate nouns.
While a feminine noun like `gatas' (female cats) always denotes female referents, the masculine `gatos' (cats) can be used for both male and non-male referents, for instance in mixed-sex groups \cite{kramer2020grammatical}. 
Different theories account for this asymmetry, either sustaining or contradicting the principle of unambiguous encoding. 
According to some accounts, the masculine gender value [Masc] in nouns like `gatos' is fully \emph{specified}, denoting either the presence of at least one male in a group or representing the referent as prototypically male. 
In this case, occurrences of a [Masc] value would correspond to presence of a \emph{+male} semantic attribute \cite{gygax2021masculine}. Then, the uneven distribution of male referents would explain the differences in use~\cite{haspelmath2017explaining, haspelmath2021explaining}. 
Another option is that only the feminine value expresses a \emph{+female} semantic attribute, while the masculine refers to entities when it is not true, relevant, or possible to specify they are female \cite{matthews1991morphology,franzon2023entropy}. 
In decoding, unambiguously retrieving the correct semantic attribute from [Masc] would not be possible: the grammatical value would then be \emph{underspecified}. In this case, the presence of a grammatical masculine value would be decoupled from the presence of a \emph{+male} semantic attribute. 
The asymmetry in encoding would depend on the communicative need to convey a specific semantic attribute in that context. 
So far, no quantitative cross-linguistic study has confirmed or refuted underspecification as a fundamental property of grammars. 

\subsubsection*{Grammar supports language processing}
Besides encoding meaning, grammatical values also play a role in language processing by tracking related words through agreement \cite{dye2017functional,franzon2023entropy,steele1978word}.
In particular, a core aspect of language processing is the ability to anticipate upcoming words, especially in comprehension~\cite{kuperberg2016we,ferreira2018integration,huettig2016individual,huettig2022parallel,futrell2020lossy, ryskin2023prediction, shain2024large}. 
As sentences unfold through time, grammatical agreement provides cues to predict these upcoming words: for instance, in a sentence a feminine determiner `la' (Spanish, the) signals that a feminine noun will follow \cite{steele1978word}. 
As experiments in several languages report, listeners use grammatical values to orient expectation of future words \cite{wicha2004anticipating, hopp2016learning, altmann2007real}, even in early stages of language acquisition \cite{fuchs2022facilitative, karaca2024morphosyntactic, parshina2024predictive}.

Linking words via agreement, or grammar's \emph{functional} role, is what guides listener expectations in predictive processing. 
As such, the distribution of grammatical values across different words has implications for processing ease~\cite{franzon2023entropy,dye2017functional}. For instance, if all nouns in Spanish were grammatically feminine, then seeing a feminine determiner `la' would give no discriminative benefit for the upcoming noun. In contrast, if feminine and masculine values are distributed uniformly (50-50) across nouns, then on average, seeing `la' halves the set of potential upcoming nouns, making prediction easier. 

Crucially, the functional role of grammatical values concerns all words in the lexicon and operates independently of their semantic content. For instance, the Spanish noun `silla' ([Fem, Sing], chair) is feminine without referring to a female entity. Here, feminine is only a \emph{formal} value which encodes no meaning. While formal values are sufficient for sustaining grammatical agreement, no languages have a grammatical system composed only of formal values: whenever a grammatical feature exists in a language, there are always words in the lexicon for which it is semantically interpreted~\cite{corbett1991gender, corbett2000number, grambank_release}. As such, the set of grammatical values used for functional purposes never exceeds the set of values that are semantically interpreted. Although this property is a known linguistic universal, its foundations are still unknown. 

\subsection*{Contributions}
Since semantic and functional aspects necessarily coexist in all known grammatical systems, a full account of grammatical organization must address both. We contribute a theoretical framework that captures their interaction in shaping grammar and its use. 
Our formalization predicts grammatical values to (1) always inherit from semantic categories and (2) exhibit semantic underspecification. 
We compare our analytical predictions with numerical simulations and large-scale data from 12 diverse languages, finding that, in reality, languages do not maximize semantic encoding but rather prioritize functional goals. 
Our framework applies to any grammatical feature, but in what follows, we focus on gender and number, which are widespread across all language families~\cite{corbett1991gender, corbett2000number, grambank_release, wals-30, wals-34}. 

\subsection*{An information-theoretic model of grammatical organization} 
We explain the empirical distribution of grammatical values as a function of communicative goals and constraints. To do this, we model the organization of grammatical systems as a multi-level optimization problem, 
from which we later derive grammatical universals.
 
Grammatical features need to satisfy semantic and functional goals. Semantically, they need to encode attributes of a referent, in a consistent way across different referents. For example, the grammatical feminine [+Fem] for any referent, like a cat (`gata'), conveys that the cat is female; ideally, the same [+Fem] consistently means \emph{+female} for words denoting other referents, like a friend (`amiga').
Now considering grammar's functional role, the distribution of grammatical values across different words has implications for processing ease. For instance, if all nouns in Spanish were grammatically masculine, then seeing a masculine article `el' (`the') would give no predictive benefit for the upcoming noun. In contrast, if the grammatical masculine and feminine are distributed 50-50 across nouns, then on average, seeing `el' halves the set of potential upcoming nouns, making prediction easier. 

We see, then, that grammatical systems are organized on two levels: grammatical values need to distribute probability mass over \emph{instances} of a given word, and across different \emph{words} in the lexical system. As such, we propose an objective, Eq.~\eqref{eq:objective}, at \emph{two scales}, the word level and the system level. 
In this section, we derive Eq.~\eqref{eq:objective}; then, in later sections, modeling languages as optimizing Eq.~\eqref{eq:objective}, we theoretically derive properties of the grammatical value distribution that are empirically attested in the world's languages.

\subsection*{Communication model}
Our model of communication shown in \cref{fig:keyfig}a is based on Shannon's communication channel \cite{shannon1948mathematical}, further detailed in the Methods. 
We assume a deterministic \emph{sender} that observes \emph{an instance of a referent} and utters a noun, mapping the referent's relevant semantic attributes into grammatical values. 
For example, in \cref{fig:keyfig}a, the sender encodes \underline{two} \underline{male} friends into the Spanish word `amigos'~[\underline{Masc}, \underline{Plur}]. For a receiver of this signal (\cref{fig:keyfig}a right) to recover a sufficient approximation of the original meaning \emph{two male friends}, 
the word `amigos' needs to be informative about the numerosity and gender of the friends. That is, given the word `amigos'~[Masc, Plur], the semantic attributes of the friends (\emph{male}, \emph{two}) should have low \emph{surprisal}, that is, the meaning should be reasonably decodable from the utterance.  
As is standard in the literature \citep{Zaslavsky_Kemp_Regier_Tishby_2018,Mollica_Bacon_Zaslavsky_Xu_Regier_Kemp_2021,gibson2019efficiency}, for simplicity, we model a language as the sender's meaning-to-signal encoding.

The sender may need to encode diverse referents (dog vs.~cat) in different instantiations (cat vs.~cats). We model the referent that the sender encodes in a single communicative instance as a random variable $X$ that follows the probability distribution $p_X(x)$. Here, $x$ is a possible realization of the referent (e.g.,~\emph{cat}), and $\mathcal X$ the set of all possible referents.

Each referent is defined by a set of $k$ independent, or orthogonal, \emph{semantic features} $\mathbf{A} = \{A_i\}_{i=1}^k$ (e.g., numerosity, sex). 
Because the $\{A_i\}_{i=1}^k$ are orthogonal, we model them separately, where each feature $A_i$ varies according to a categorical distribution $p_{A_i|X=x}$ over $|A_i|$ possible \emph{semantic attributes} 
(e.g., the attribute female for the semantic feature of sex).\footnote{The set of features and attributes in $\mathbf{A}$, as well as their distributions among referents, can vary across environmental and cultural contexts. Here we focus on the categorizations most commonly found in natural languages \cite{grambank_release}, and examine how their possible distributions relate to semantic encoding. Encoded attributes—particularly gender—are not necessarily constrained to binary or ternary distinctions of sex, as seen in many languages \cite{corbett1991gender}. While we use this simpler case for illustration, the model is flexible and can be extended to other semantic features. Importantly, the distribution of referents, as seen by speakers, does not necessarily mirror real-world distributions in a strict way and can be filtered through perception. For example, a speaker may not know the sex of the cat they are referring to—in such cases, the encodable attribute would be 'other'. It is not possible to realistically measure this aspect, so we do not consider it here.}
We model the sender as an \emph{encoder} $f$ mapping the referent's semantic features $\mathbf{A}$ to $k$ corresponding \emph{grammatical features} $\mathbf{W_A} = \{W_{A_i}\}_{i=1}^k$ (see \cref{fig:keyfig}b for an illustration of $A=$ gender). 
Each $W_{A_i}$ follows a categorical distribution over $|W_{A_i}|$ \emph{grammatical values}. While the overall grammatical system is all grammatical features $\mathbf W_A$, we restrict our analysis to each $W_A$ separately due to orthogonality.

The distribution of grammatical values fully depends on the meaning distribution $p_X$, the attribute distributions $p_{A|X}$, and encoder $f$. Our goal is to explain the distribution of $W_A$ via $f$, if $f$ satisfies certain properties for efficient communication.\footnote{Unlike prior work \cite{Zaslavsky_Kemp_Regier_Tishby_2018,Mollica_Bacon_Zaslavsky_Xu_Regier_Kemp_2021}, we assume lossless source encoding; that is, we do not model variation in the sender's perception of the referent, due to the impossibility to disentangle the true and perceived distribution of semantic attributes given the referents and words used to refer to them.}

\subsection*{Objective}
We now motivate the objective in \eqref{eq:objective}, considering the properties $f$ needs to satisfy at the level of the communicative instance, then at the level of the system. \eqref{eq:objective} is a multi-level optimization problem that should be read bottom-up as follows. We start with the set $\mathcal F$ of all possible encoding strategies $f$. First, the instance-level optimization in \eqref{eqn:word_level} is solved, cutting candidate encoders to a subset $\mathcal G \subset \mathcal F$. Finally, the system-level optimization problem in Eqs.~(\ref{eq:system_level}) and (\ref{eq:consistency}) is solved to yield the optimal solutions $f^*$. In what follows, we derive \eqref{eq:objective} on a high level. Our formulation relies on concepts from information theory \citep{shannon1948mathematical} such as Shannon entropy $\mathcal H(\cdot)$; see Materials and Methods for formal definitions. For further mathematical details, see the Appendix.

\begin{tcolorbox}[title=Grammatical organization objective, colframe=blue!60!black, colback=blue!10,left=0mm, top=0mm]
\begin{subequations}\label{eq:objective}
\begin{align}
\underset{f}{\text{max}} & \qquad \underset{\text{discriminability}}{\mathcal H(W_A)}  \label{eq:system_level}\\
\text{s.t.} 
& \qquad f \in \underset{{f\in \mathcal G}}{\text{argmin}} \left \{\underset{\text{size}}{|W_A|} + \underset{\text{consistency}}{\beta \mathcal H (W_A | A)} \right \} \label{eq:consistency} \\
\text{where} \nonumber \\
\mathcal G = & \underset{x \in \mathcal X}{\bigcap} \underset{f \in \mathcal F}{\text{argmin}} \left\{ \underset{\text{memory}}{\mathcal H(W_A|X\text{=}x)}+\alpha_x \underset{\text{surprisal}}{\mathcal H(A|W_A,X\text{=}x)} \right\} \label{eqn:word_level} \\
& \qquad f: A \mapsto W_A \label{eqn:f_definition}\\
&  \qquad \beta \geq 0 \label{eqn:beta} \\
& \qquad \alpha_x \geq 0 \ \ \forall x \in \mathcal X \label{eqn:alpha} 
    \end{align}
\end{subequations}
\end{tcolorbox}

\subsubsection*{Semantic encoding and compression at the instance level} We first consider the instance-level objective. Here, $f$ needs to encode information about the semantic features of the instance of a referent. At an extreme, $f$ can maximally encode all information into grammatical values. This results in low surprisal of the semantic attribute $A$ given grammatical value $W_A$, but this strategy incurs a high cost.

This memory-surprisal tradeoff \cite{Futrell_Hahn_2022} formalizes Zipf's Principle of Least Effort (PLE) \cite{zipf1949human}, which states that, all else equal, one chooses the least costly strategy to solve a problem. This tradeoff takes place for each noun, given by \eqref{eqn:word_level}. For a given $X=x$ (e.g., `cat'), for each semantic feature $A$ (e.g., \emph{sex}), the encoding $f$ minimizes the cost $\mathcal H(W_A|X\text{=}x)$
while maintaining low average surprisal $\mathcal H(A| W_A, X\text{=}x)$ (hereon \emph{surprisal}), see Materials and Methods for details. 
Keeping surprisal low reflects a \emph{communicative goal}: a sender says $W_A$ so that a receiver can approximately decode it back into $A$. On the other hand, minimizing cost is a \emph{cognitive goal}: the sender wants to minimize effort to store and transmit information. The parameter $\alpha_x \in \mathcal A = [0, \infty)$ modulates this tradeoff for each referent $x$, proxying communicative need. The higher $\alpha_x$ is, the more important it is to communicate precisely about $x$. We detail in a later section that $\alpha_x$ determines whether $A$ is fully, partially, or not encoded into grammar for the referent $x$.

\subsubsection*{Semantic consistency at the system level} 
Because the grammatical system needs to be easy to use, grammatical values  span a small set of values and correlate across words. Otherwise, a viable system could have, for instance, $N$ separate values to encode plural, one for each of $N$ referents; such a system would be semantically unambiguous but overly complex. The pressure for consistency instead superimposes a global coherence \emph{across} referents.
The global consistency between semantic attribute and grammatical value enters as the constraint \eqref{eq:consistency} in \eqref{eq:objective}, modulated by parameter $\beta \geq 0$ (\eqref{eqn:beta}). While at the instance level we minimize surprisal of the semantic attribute $A$ given $W_A$, at the system level we minimize the entropy of the grammatical value $W_A$ given $A$. To illustrate the difference, the former requires that the grammatical singular be the same in all instances of $x=$ `cat'; the latter requires that the singular be the same for both `cat' and `chair'. 

At the same time, the size of the system $|W_A|$, namely the number of values it comprises, needs to be manageable. Again realizing Zipf's PLE, this is written as a size cost in \eqref{eq:consistency}. 

\subsubsection*{Discriminability at the system level} 
Finally, grammatical values serve a functional role, aiding agreement-based discriminability \cite{dye2017functional, franzon2023entropy}. \eqref{eq:system_level} expresses that an encoder that optimizes for agreement-based discriminability distributes grammatical values uniformly across the lexicon, maximizing entropy $\mathcal H(W_A)$ of the system. Thanks to grammatical agreement, the max-entropy distribution over grammatical values optimally reduces competition over the next word in online processing.

To see why this is the case, consider as an agreement-based discriminability measure {Agr}D, defined as equal ($\triangleq$) to the proportion of competitors cut off: 
\begin{equation}
\label{eq:agrd}
    \text{Agr}D \triangleq 1 - \sfrac{1}{2^{\mathcal H(W_A)}}, 
\end{equation}
where higher entropy $\mathcal H(W_A)$ means greater discriminability ($\text{Agr}D \to 1$), and $0$ entropy means no discriminability ($\text{Agr}D \to 0$)~\cite{franzon2023entropy}. This pressure favoring agreement-based discriminability is given in \eqref{eq:objective} by the entropy term $\mathcal H(W_A)$ over grammatical values.

Note that the instance level and system level objectives each contain an entropy term that is respectively minimized and maximized. There is a key difference between minimizing $\mathcal H(W_A|X=x)$ at the instance level (\eqref{eqn:word_level}) and maximizing $\mathcal H(W_A)$ at the system level (\eqref{eq:system_level}). Minimizing entropy given $X=x$ concentrates probability mass on one grammatical value, meaning that all the instances of a word denoting a referent can take only one grammatical value. Meanwhile, at the system level, maximizing entropy evenly disperses probability mass across the lexicon. This distinction is depicted in \cref{fig:keyfig}c: in Spanish, most referents' grammatical value distributions (ellipses) concentrate to either [Masc] or [Fem], but overall probability mass across referents distributes evenly between the two grammatical genders.

\section*{Deriving semantic universals of grammatical systems}
From the minimal set of cognitive pressures given in \eqref{eq:objective}, it is possible to derive known universal relationships between semantic attributes and grammatical encoding. We assume that languages are optimal with respect to \eqref{eq:objective} for some $\mathbf{\alpha}$, $\beta$. Then, we derive theoretical bounds regarding the distribution of grammatical values, showing that (1) all grammatical values are semantically interpretable in at least one part of the lexicon; (2) the memory-surprisal tradeoff underlies the extent of semantic underspecification. This section only contains theoretical results; we validate each one with numerical simulations, see \emph{Validating Analytic Results with Simulations} in the Appendix.

\subsection*{Grammatical values provably inherit from semantics}
In every language, grammatical values serve a functional role. At the same time, they also encode a semantic attribute at least in some part of lexicon \cite{wals-32,corbett2012features}.  
As functional aspects operate independently from semantics, it is still unknown why grammatical values universally \emph{inherit} from a semantic foundation, whether that be numerosity, animacy, or other properties. In this section, we formally define semantic inheritance in grammatical systems and show that this observed universal mathematically arises when languages are optimal with respect to~\eqref{eq:objective}. 

What do we mean by \emph{semantic inheritance}? We say a grammatical feature $W_A$ inherits from semantic attribute $A$ if, for some part of the lexicon, $W_A$ is \emph{informative} about $A$. An example of this is grammatical gender in Romance or Slavic languages, where for some animate nouns, grammatical gender ($W_A$) marks sex-related attributes ($A$). We state this definition below in information-theoretic terms:

\begin{definition}[Semantic inheritance]
\label{def:inheritance}
    Let $W_A$ be some grammatical feature. $W_A$ \emph{inherits} from semantic attribute $A$ if $\exists$ non-empty $\mathcal X_{A} \subset \mathcal X$ for which $\mathcal H(A\mid W_A,x\in \mathcal X_A) < \mathcal H(A \mid x\in\mathcal X_A)$.
\end{definition}

In general, for two random variables $A$ and $B$, conditional entropy $\mathcal H(A \mid B) \leq \mathcal H(A)$, that is, learning about $B$ will increase or maintain our knowledge about $A$ on average \cite{cover_thomas}. \Cref{def:inheritance} instead requires that knowing the grammatical feature $W_A$ \emph{strictly increase} our knowledge about $A$. If languages are optimal with respect to \eqref{eq:objective}, then their grammatical features necessarily display semantic inheritance, yielding our first result: 

\begin{theorem}[Grammatical values inherit from semantics]
\label{thm:inheritance}
    If $f$ optimizes \eqref{eq:objective}, then each $W_A$ either inherits from some $A$ or consists of a single neutralized value for the whole lexicon.
\end{theorem}
\begin{proof}
    See Appendix.
\end{proof}

\Cref{thm:inheritance} states that semantic inheritance follows from optimizing \eqref{eq:objective}. But, beyond $W_A$ simply encoding $A$ in some part of the lexicon (\Cref{thm:inheritance}), we can further show that $|W_A| \leq |A|$, or that there are never more grammatical values than perceptual categories for a given semantic attribute. For instance, if we perceive $A=\text{gender}$ to vary over $\{+male, +female, \o\}$, then our theory predicts \emph{at most} three grammatical gender values. 
To prove this, we begin by showing the number of grammatical values are constrained by semantic attributes, separately for each referent (\Cref{lemma:num_feats}). Then, given \Cref{lemma:num_feats}, we show that the same statement holds for the entire lexicon (\Cref{thm:num_feats}).

\begin{lemma}[Number of grammatical values given referent]
\label{lemma:num_feats}
    An $f^*$ optimizing \eqref{eq:objective} has
    $$|W_{A|X=x}| \leq |A_{X=x}|.$$
\end{lemma}
\begin{proof}
    See Appendix.
\end{proof}

We illustrate Lemma \ref{lemma:num_feats} by considering two examples from a common gender system \cite{corbett1991gender}. In our examples, $A_{\text{gender}}$ varies over $\{+male, +female, \o\}$, corresponding to male, female, and absence of biological sex or gender.\footnote{Gender systems can also encode other values, including culturally defined categories \cite{corbett1991gender}. As the perceptions of gender and knowledge of biological sex difference evolve, so does the set of possible categories in a system. We leave a deeper study of this phenomenon to future work and focus here on the distinctions supporting the most frequent systems found in languages~\cite{wals-31}.}
First, let $x$ be `chair', which has a biological sex of $\o$; $|A_{X=\text{chair}}| = 1$, and our model correctly predicts for Spanish \emph{silla} to take on $|W_{A|X=\text{chair}}| \leq 1$ grammatical values if $f$ optimizes \eqref{eq:objective}. Second, if $x$ is ``cat'', whose biological sex takes values in $A_{X=\text{cat}} = \{+male, +female\}$, our model correctly predicts the word \emph{gato} to inflect for a maximum of $|W_{A|X=\text{cat}}| \leq 2$ grammatical values. An example for which the inequality is strict, that is, $|W_{A|X=x}|<|A_{X=x}|$, happens for numerosity: humans perceive a rich set of numerical distinctions and quantities (\textgreater $4$)
~\cite{dehaene2011number,butterworth2018introduction}, but most languages only grammatically encode $2$ categories, singular and plural.

We just showed that the number of grammatical values is bounded by the number of perceptual categories for each referent. Now, we show that this also holds over the system of referents. Note that the instance level objective alone would permit systems where, e.g., each referent has a unique plural value (``cat"-``cats", ``dog"-``dogss", ``bat"-``batsss"). So, only given the instance level objective, the overall $|W_A|$ can still be greater than $|A|$ when aggregated over all $x$. The system level simplicity term $\min_{f\in \mathcal G} |W_A| + \beta \mathcal H(W_A|A)$ will instead constrain $|W_A| \leq |A|$, yielding the following theorem:

\begin{theorem}[Number of grammatical values for a feature in system]
\label{thm:num_feats}
    If $f^*$ satisfies \eqref{eq:objective}, then the number of values of $W_A = f^*(A)$ is bounded above by the number of values of $A$:
    $$|W_A| \leq |A|.$$
\end{theorem}
\begin{proof}
    See Appendix.
\end{proof}

\Cref{thm:num_feats,thm:inheritance} demonstrate that optimizing for basic information-theoretic constraints produces grammatical systems that inherit from, and are bounded by, semantic categories. As an example, consider again the gender system $\{+male, +female, \o \}$, with $\o$ corresponding to absence of biological sex or gender. Empirically, in all languages whose grammatical gender system is historically based on this semantic categorization, there are at most three values: masculine [Masc], feminine [Fem], and neuter [Neut]~\cite{corbett1991gender}. Our proposal provides an analytical basis for empirical universals \cite{corbett1991gender,corbett2012features} in which grammatical systems universally derive from semantic assignment.

\subsection*{Memory-surprisal tradeoff underlies semantic specification}
\label{sec:specification}
We just showed that grammatical values inherit from semantic features. But the \emph{degree} of inheritance, that is, how informative the grammatical value is about semantics, can vary. Here, we break down the variation in informativity in terms of semantic specification, showing that underspecification in grammatical systems mathematically results from \eqref{eq:objective}. More precisely, semantic specification can be divided into three categories: (1) full specification, in which grammar perfectly encodes semantics, (2) underspecification, in which grammar partially encodes semantics, and (3) neutralization, in which grammar does not encode semantics at all. 
These three categories correspond to three sets of encoding strategies $f$, where full specification and neutralization reflect opposite extremes of underspecification.
The category $f$ falls into is governed by $\alpha_x$ in \eqref{eq:objective}.  

\begin{tcolorbox}[ams equation,title={Semantic encoding is a partition of $\mathcal A$},]
\label{eq:semantic_encoding}
    f_\alpha(x) 
    \begin{cases}
        \text{full specification}, & \text{if } \alpha_x \to \infty \\
        \text{underspecification}, & \text{if } 0 < \alpha_x < \infty \\
        \text{neutralization}, & \text{if } \alpha_x = 0
    \end{cases}
\end{tcolorbox}

The parameter $\alpha_x$ may be interpreted as the degree of communicative need. When $\alpha_x \to \infty$, it is paramount to perfectly communicate about $x$. In this case, full specification minimizes surprisal and $f$ is a one-to-one mapping from grammatical to semantic value. The encoding strategy of neutralization lies at the other extreme which maximizes simplicity ($\alpha_x = 0$, lowest communicative need): neutralization, which always produces the same value, e.g., [Masc], for a given referent, minimizes memory complexity for the gender feature by placing all probability mass on a single value. 
Semantic underspecification occurs in all other cases, namely when $0 < \alpha_x < \infty$. If $f$ is subject to both surprisal and memory costs minimization pressures, then the grammatical value distribution differs from the semantic attribute distribution given the referent: surprisal minimization anchors $p_{W_A | X=x}$ to $p_{A|X=x}$, while memory  minimization compresses $p_{W_A | X=x}$ to one value. Then, the mapping $f$ compresses the semantic values of $A$ to the grammatical values of $W_A$. This means that there exists a grammatical value which necessarily has a higher probability than a semantic value that maps to it. These so-called \emph{default} values occur in semantically underspecified noun systems, in the way described in \cref{fig:keyfig}b. To illustrate, in Spanish, when referring to a cat, `el gato'[Masc] may be employed by the sender
as the default value although the biological sex distribution is balanced.

Finally, it is possible that a feature $A$ takes some semantic attribute with probability 1; while cats vary in their biological sex, inanimate referents like chairs do not, taking a biological sex of $\o$. In this case, $\mathcal H(A| W_A,X=x)$ and $\mathcal H(W_A|X=x)$ are both minimized when $W_A$ is a constant for each referent, regardless of $\alpha_x$. Formal values like in `silla' [Fem] are values for which $\alpha_x = 0$, including also nouns of animate referents that only occur in one gender, as Spanish `camello' [Masc] (camel) and `jirafa' [Fem] (giraffe).

Now, we claim that semantic underspecification is a consequence of the memory-surprisal tradeoff when $\alpha_x < \infty$. Again, let languages optimize \eqref{eq:objective}. We argue that semantic underspecification occurs when the entropy of the semantic attribute distribution bounds that of grammatical values:

\begin{lemma}[Entropy of grammatical values given referent]
\label{lemma:entropy_feat_ref}
    The entropy of $W_A = f^*(A)$, where $f^*$ satisfies \eqref{eq:objective}, is bounded above by the entropy of semantic feature $A$:
    $$\mathcal H(W_A|X\text{=}x) \leq \mathcal H(A | X=x)$$
    The equality condition occurs when surprisal $H(A|W_A,X\text{=}x)$ is minimized.
\end{lemma}
\begin{proof}
    See Appendix.
\end{proof}

If the entropy of the grammatical values is less than that of semantic attributes, then the encoding $f$ necessarily shifts probability mass from several semantic attributes onto one grammatical value (see \cref{fig:keyfig}b; $f$ maps some +\emph{female} and all +\emph{male} referents to [Masc]). If a grammatical value encodes several semantic attributes, then it is semantically underspecified by definition. Thus, underspecification is a result of minimizing memory cost at the instance level (\eqref{eqn:word_level}).

\subsubsection*{Semantic specification in real languages}
We have just proposed semantic specification to theoretically follow from the memory-surprisal tradeoff. Now, we show empirically that patterns of semantic underspecification are consistent across a sample of 12 languages of diverse groups. The languages considered are French, Spanish, Catalan, and Italian (Romance), German, English, Swedish, and Dutch (Germanic), Polish and Slovene (Slavic), and Hebrew and Arabic (Semitic), selected due to availability of large-scale corpora. We analyze their grammatical distributions using web-scale corpus data, further detailed in the Methods, and here situate them with respect to the specification strategies in \eqref{eq:semantic_encoding}. Distributions of grammatical features are reported in \cref{tab:entropies} and \cref{fig:token_distribution}.

All considered languages mark number in nouns, and, except for English, all mark gender. The vast majority of referents in these languages each take one grammatical gender. Inanimate nouns $x$ have a gender or biological sex of $\o$, so $\alpha_x$ in \eqref{eqn:word_level} can take any value in $\mathcal A$; for animate referents that have gender or biological sex but take only one grammatical gender, such as `el camello' 
[Masc] (camel in Spanish), $\alpha_x = 0$, reflecting maximal instance-level compression in \eqref{eq:semantic_encoding}. 

When grammatical gender is semantically interpreted, $\alpha_x > 0$. Animate nouns that inflect for gender in Romance, Semitic, and Slavic languages fall into this category. The distribution over grammatical values is surprisingly consistent within and between language families, where 
\begin{equation}
    p(\text{Masc,Sing}) \geq p(\text{Masc,Plur}) > p(\text{Fem,Sing}) > p(\text{Fem,Plur})
\end{equation} 
where distributions are shown on the right side of plots in \cref{fig:token_distribution}. If male and female referents are evenly distributed among this animate subset,\footnote{We make this assumption considering \emph{contemporary} language use; the assumption holds for animals, whose biological sex is 50/50, and may be reasonable for nouns denoting jobs and professions, given 2024 estimates of women's labor participation by the World Bank~\citep{Ortiz-Ospina_Tzvetkova_Roser_2024}.} then the systematic preference for the grammatical masculine would indicate \emph{semantic underspecification} of the masculine.

For the considered languages with grammatical gender, we summarize the following categories in \cref{tab:partition_real_data}: $\alpha_x = 0$ corresponding largely to neutralized animate nouns; $\alpha_x > 0$ corresponding largely to animate nouns that inflect for gender; $\alpha_x \to \infty$ in Dutch nouns that encode \emph{animacy}, referring to humans, for example.

Finally, in all languages considered, grammatical number semantically encodes numerosity. Human numerical cognition can distinguish precisely up to four units, and map approximately larger sets~\cite{cheyette2020unified}. While grammatical values encoding these numerosities are found across languages, no grammatical system displays all of them, hence $0 < \alpha_x< \infty$ for all referents inflecting for number in all languages. 
    
\begin{table*}[]
\small
    \centering
\begin{tabular}{cccc}
    \toprule
        $\alpha_x$ & languages & subset of lexicon & examples \\
        \midrule
        $\alpha_x \to \infty$ & Dutch & nouns referring to humans & \emph{de leraar} (Dutch) \\
        $\alpha_x > 0$ & \textbf{Romance, Slavic, Semitic,} German & animate nouns that inflect for gender & \emph{il gatto, la gatta} (Italian) \\
        $0 < \alpha_x < \infty$ & Dutch, Swedish & $\emptyset$ & -- \\
        $\alpha_x = 0$ & \textbf{Romance, Slavic, Semitic,} German & neutralized animate nouns & \emph{el camello} (Spanish) \\
        \hline
    \end{tabular}
        \caption{Values of $\alpha_x$ correspond to particular subsets of the lexicon in the 12 languages considered. From left to right, the columns consider the range of $\alpha_x$; the language or language \textbf{family}, the subset of the lexicon to which $\alpha_x$ applies for those languages, and several examples. All languages except for English are considered for grammatical gender and all languages are considered for number. In general, nouns for which grammatical gender does not encode semantics can take any value of $\alpha_x$.}
        \vspace{-1.5ex}
    \vspace{-1ex}
    \label{tab:partition_real_data}
\end{table*}
We showed a minimal set of cognitive principles to explain why grammatical systems inherit from semantic categories and exhibit semantic underspecification. In particular, our results suggest that pressures from semantic encoding constrain the distribution of grammatical values. Now, we turn to the other key pressure in our model, that is, the functional goal to optimize agreement-based discriminability. We find that semantic and functional pressures interact to shape the distribution of grammars in real languages.

\begin{figure*}[t]
   \centering
   \hspace{-1.7em}
    \includegraphics[width=0.95\textwidth]{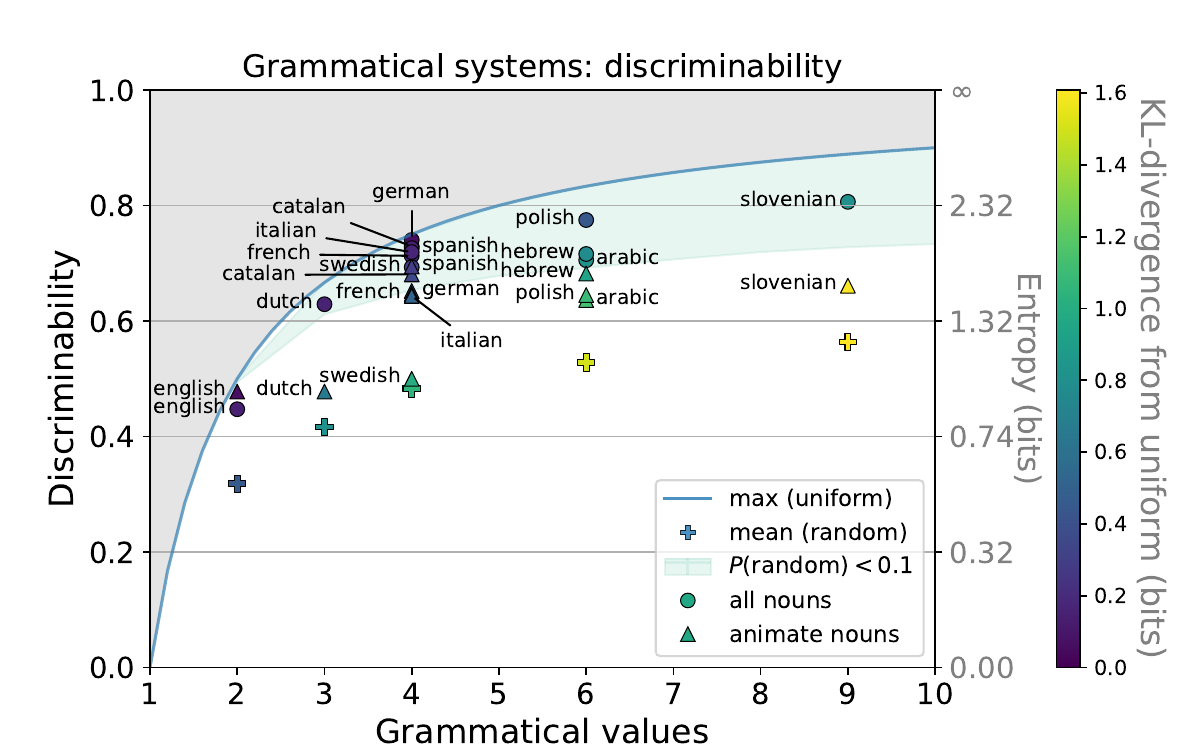}
    \caption{Discriminability (left y-axis) and entropy (right y-axis) of the overall (gender, number) systems for observed languages with respect to number of grammatical values (x-axis). The optimal solution for functional processing, which is the uniform distribution, is the blue curve; above it, the gray shaded area comprises theoretically unreachable values. Each plotted point is either the set of all nouns of a language (\textcolor{teal}{$\bullet$}) or its animate subset (\textcolor{teal}{$\blacktriangle$}). Points are colored by their KL-divergence ($\mathcal D_{KL}$)
    (bits) from the uniform distribution, which is their vertical distance to the optimal line. Mean discriminability and entropy are plotted (\textcolor{teal}{$\mathord{\text{\ding{58}}}$}) for a baseline class of hypothetical languages sampled i.i.d.~($N$=$1000$) from a Dirichlet distribution. For this baseline class, the empirical $\geq$90th percentile for discriminability (entropy) is shaded in light green. For all languages except English (whose animate and all-discriminability are roughly equal), the discriminability of all nouns is higher than that of the animate subset. All languages' discriminability and entropy are close to the optimum \emph{unlikely by chance} ($p<0.1$), seen by plotted points (circles) falling in the shaded region. }
    
    \label{fig:optimal_discrim}
\end{figure*}

\subsection*{Languages are near-optimal for functional processing}
Beyond semantic encoding, grammars ease language processing by enabling word prediction through agreement. 
It remains unknown whether cross-linguistically, grammatical systems optimize for agreement-based discriminability. 
In this section, we present evidence demonstrating that the grammatical number and gender organization of our sample of 12 languagesare near-optimal for functional processing. In addition, we consider subsystems where grammatical values are semantically interpreted; in these subsystems, we find a negative interaction between semantic and functional goals, where the need to encode meaning pulls grammatical distributions away from optimality.

To assess how optimal each system is, we compare it to a baseline class of random languages sampled from a Dirichlet distribution (details in Appendix).
The extent to which grammars deviate from optimal functional processing is given by their Kullback-Leibler divergence (KL-divergence or $\mathcal D_{KL}$) from the uniform distribution (see Materials and Methods for definition). We consider a system to be near-optimal if its $\mathcal D_{KL}$ is significantly lower than that of the random baseline, with a p-value cutoff of $0.1$.

\Cref{fig:optimal_discrim} shows the entropy $\mathcal H(\mathbf{W_A})$ and agreement-based discriminability $\text{Agr}D$ of the number and gender grammatical system as a function of grammatical values, plotting all languages.  
We find that the system entropy $\mathcal H(\mathbf{W_A})$ of all languages is near-optimal for 
agreement-based discriminability, seen in the figure by points \textcolor{teal}{$\bullet$}, each standing for a language's whole lexicon, falling in the shaded region ($p<0.1$) under the optimal curve (blue). 
In the same figure, we plot the subset of animate nouns \textcolor{teal}{$\blacktriangle$} for which grammatical gender is semantically interpreted. In contrast, the animate subsets are less optimal for agreement-based discriminability (further away from the blue curve in \cref{fig:optimal_discrim}) and do not significantly differ from random languages (falls outside the shaded region). This suggests that optimal discriminability is salient at the level of the entire system, but not for subsystems where semantic encoding is preferred. 

\begin{small}
\begin{table*}[t]
\centering
\begin{tabular}{@{}lccccccccccccc@{}}
\toprule
& \multicolumn{4}{c}{\textbf{Gender}} & \multicolumn{4}{c}{\textbf{Number}} & \multicolumn{4}{c}{\textbf{Overall}} \\ \cmidrule(lr){2-5} \cmidrule(lr){6-9} \cmidrule(lr){10-13}
 Language & \# Values & Max $\mathcal H$ & $\mathcal H$ &$\mathcal D_{KL}$ & \# Values & Max $\mathcal H$ & $\mathcal H$ & $\mathcal D_{KL}$ & \# Values & Max $\mathcal H$ & $\mathcal H$ & $\mathcal D_{KL}$ \\ \midrule 
Catalan & $2$ & $1.0$ & $\approx 1.0$  & $3e$-$4^\ddagger$ & $2$ & $1.0$ & $0.89$ & $0.11$ & $4$ & $2.0$ & $1.89$ & $0.12^\ddagger$ \\
French & $2$ & $1.0$ & $\approx 1.0$ & $3e$-$3^\ddagger$ & $2$ & $1.0$ & $0.80$ & $0.23$ & $4$ & $2.0$ & $1.79$ & $0.23^\ddagger$ \\
Italian & $2$ & $1.0$ & $0.99$ & $0.01^\dagger$ & $2$ & $1.0$ & $0.83$ & $0.17$ & $4$ & $2.0$ & $1.83$ & $0.18^\ddagger$ \\
Spanish & $2$ & $1.0$ & $\approx 1.0$ & $3e$-$3^\ddagger$ & $2$ & $1.0$ & $0.87$ & $0.14$ & $4$ & $2.0$ & $1.87$ & $0.14^\ddagger$ \\
\midrule
Polish & $3$ & $1.58$ & $1.46$ & $0.15^\dagger$ & $2$ & $1.0$ & $0.80$ & $0.23$ & $6$ & $2.58$ & $2.15$ & $0.43^\ddagger$ \\
Slovene & $3$ & $1.58$ & $1.54$ & $0.05^\ddagger$ & $3$ & $1.58$ & $0.86$ & $1.61$ & $9$ & $3.17$ & $2.37$ & $1.70^\ddagger$\\
\midrule
Arabic & $2$ & $1.0$ & $0.91$ & $0.09$ & $3$ & $1.58$ & $0.86$ & $1.91$ & $6$ & $2.58$ & $1.76$ & $0.82^\dagger$ \\
Hebrew & $2$ & $1.0$ & $0.97$ & $0.03^\dagger$ & $3$ & $1.58$ & $0.85$ & $2.49$ & $6$ & $2.58$ & $1.81$ & $0.77^\dagger$\\
\midrule
German & $3$ & $1.58$ & $1.51$ & $0.08^\ddagger$ & $2$ & $1.0$ & $0.80$ & $0.23$ & $4$ & $2.0$ & $1.94$ & $0.06^\ddagger$ \\
English & \o & -- & -- & -- & $2$ & $1.0$ & $0.86$ & $0.16$ & $2$ & $1.0$ & $0.86$ & $0.16$ \\
Dutch & $2$ & $1.0$ & $0.88$ & $0.13$ & $2$ & $1.0$ & $0.74$ & $0.30$ & $3$ & $1.58$ & $1.44$ & $0.15^\dagger$\\
Swedish & $2$ & $1.0$ & $0.83$ & $0.19$ & $2$ & $1.0$ & $0.91$ & $0.09$ & $4$ & $2.0$ & $1.71$ & $0.39^\ddagger$ \\
\bottomrule
\end{tabular}
\caption{Number of possible values, entropy $\mathcal H$ of grammatical values in bits, and the KL-divergence $\mathcal D_{KL}$ from the optimal uniform entropies (max~$\mathcal H$) are reported (left to right) for \textbf{gender} and \textbf{number} features, as well as for the \textbf{overall} grammatical system, for each language. $\mathcal D_{KL}$ is marked if significantly lower than that of random categorical distributions at a p-value cutoff of $0.05^{\ddagger}$, $0.1^\dagger$, and unmarked otherwise. The maximum entropy column is the theoretical entropy of the uniform distribution, while the entropy column is computed from corpora. From top to bottom, languages are grouped into Romance, Slavic, Semitic, and Germanic language families. }
\vspace{-4ex}
\label{tab:entropies}
\end{table*}
\end{small}

We find that for number and animate noun gender, which encode semantics, $\mathcal D_{KL}$ is high and not significantly different from random.

First, we consider grammatical gender, which is semantically interpreted for certain animate nouns in German and Romance, Semitic, and Slavic languages, but not for general nouns. For the languages considered, in all nouns, $\mathcal D_{KL}$ of gender from the optimal uniform distribution is significantly lower than chance: $0.06 \pm 0.06$ (SD) bits on average (see the left section in \cref{tab:entropies}). In particular, Romance languages' gender systems are nearly optimal for discriminability, with $\mathcal D_{KL} \approx 0$. In contrast, in the animate subset, $\mathcal D_{KL}=0.41 \pm 0.16$ across languages, significantly higher than in the set of all nouns. This suggests an opposing interaction between semantic encoding and functional processing on the shape of the distribution. When a grammatical feature does not transmit semantic information, its values disperse uniformly for functional processing; the need to transmit meaning instead pulls its distribution away from the uniform.

Analogously, we consider grammatical number, which encodes numerosity. We first note that number greatly compresses numerosity: among the observed languages, the maximum information content is only $1.58$ bits (Slovene), corresponding to $2^{1.58}\approx 3$ effective values to encode all numerosities we can perceive. $\mathcal D_{KL}$ from the optimal uniform distribution, see middle of \cref{tab:entropies}, is $0.63 \pm 0.81$ bits and not statistically lower than random, suggesting that number, like gender in animate nouns, distributes suboptimally for functional processing. Here again, the $\mathcal D_{KL}$ to the optimal distribution for discriminability marks a contrast between semantic and functional roles.
We just showed that, in the entire lexicon, grammatical values distribute nearly optimally to support functional processing; but, within each semantically interpretable subsystem (number and animate noun gender), that the grammatical value distribution is highly non-uniform. Overall, results suggest that semantic goals shape grammatical value distributions within semantically-interpreted subsystems; at the same time, functional goals disperse values across the entire lexicon.
 
 \section*{Discussion} 
Our information-theoretic model shows that universally shared grammatical structures and word distributions emerge from jointly optimizing sentence processing demands and the surprisal–memory tradeoff, modulated by communicative need. 

We extend previous formalizations focused on meaning encoding~\cite{Zaslavsky_Kemp_Regier_Tishby_2018,Mollica_Bacon_Zaslavsky_Xu_Regier_Kemp_2021,Kemp_Regier_2012} by incorporating system-level processing objectives~\cite{dye2017functional,franzon2023entropy} and a communicative need parameter, predicting encoding choices in single communicative instances.
Our results go beyond identifying  the optimal set of values in a system, and also explain their usage patterns, revealing the foundations to three key phenomena in grammatical organization—underspecification, semantic inheritance, and agreement-based discriminability—which we validate on large-scale data of 12 diverse languages. 

\subsection*{Semantic encoding and underspecification}
Previous research demonstrated that grammatical systems are organized to optimally encode semantic attributes~\cite{Mollica_Bacon_Zaslavsky_Xu_Regier_Kemp_2021}. 

Our framework introduces the influence of communicative need in semantic encoding, shaping both the overall grammatical system and the selection of values in individual communicative instances. 
Optimal encoding spans a continuum from maximizing informativeness to minimizing cost, and communicative need regulates where encoding strategies are placed on this continuum. Low communicative need corresponds to instances when speakers do not know or do not want to transmit an attribute: in this contexts, they tend to favor low-cost, less precise encoding. We predict—and confirm in real language data—that grammatical values are not always used to convey maximum semantic precision; instead, at least one value typically functions in an underspecified, general role.

In particular, underspecification consistently appears in the distribution of gender values within animate nouns, where the more frequent use of masculine corresponds to an ambiguous semantic coding, as depicted in~\cref{fig:keyfig}b, which indicates that grammatical masculine does not necessarily encode a \emph{+male} semantic attribute. 
Our model does not explain why the underspecified value is consistently masculine across languages, but it predicts that underspecified values occur at higher frequency, and in turn that a more frequent value is more likely to be interpreted as underspecified. 
A higher number of male referents denoting human occupations, perhaps due to historical and cultural drivers, can bias the frequency towards the masculine value~\cite{haspelmath2021explaining,harrison-etal-2023-run}. 
By satisfying the pressure towards consistency, the grammatical masculine may then be interpreted and reused as underspecified in the whole system.
Our perspective of underspecification suggests that, beyond the true distribution of semantic attributes, the use of a value can be influenced by its own frequency. 
This view bridges vocabulary-based theories of meaning encoding with theories causally relating word frequency to processing ease \cite{zipf1999psycho, kanwal2017zipf, piantadosi2014zipf}. Our data reflect a single moment in linguistic evolution, which can be expanded by evolutionary modeling to address how underspecification arises over time. 

\subsection*{Semantic inheritance}
Although grammatical values do not maximize meaning transmission, they inherit from semantic attributes as a foundation. This aspect was overlooked by previous accounts deriving language efficiency from surprisal-memory tradeoffs~\citep{Futrell_Hahn_2022}. 
For instance, past formalizations that minimize surprisal bound to a cost constraint include optimal solutions that admit many synonyms for one meaning \citep{jbb_2024,strouse_deterministic_2017}. 
Instead, our \eqref{eq:objective} bounds the range of grammatical values by the number of semantic attributes, suggesting that Zipf's PLE can explain a universally occurring property for known languages \cite{corbett1991gender, corbett2000number}. 

Our model takes for granted that some attributes, and not others, are grammaticalized cross-linguistically. We take the view that the chosen encoded attributes would enter in Objective~\ref{eq:objective} as another inner optimization constraining $\mathbf A$ based on extra-linguistic environmental factors. 
We leave to further studies how the perceptual and cultural salience of the grammaticalized attributes, or their distribution across referents in real-world speakers'experience could select best candidates for grammatical encoding. 

\subsection*{Functional role and discriminability}
Finally, we show for 12 languages that number and gender values distribute to optimize discriminability through agreement, enabling regular cues for word prediction. 
Our analysis can be applied to other grammatical features like tense, aspect, mood, nominal classifiers, and to the organization of functional words like prepositions, as well as lexical words, especially when they have a role in prediction, like prenominal adjectives~\cite{hahn2018information, dye2018alternative}. In these cases, the different balance between semantic and functional goals may result in different distributions that need to be empirically measured. 

\subsection*{Distance from optima and language change}
While most of the languages we examined align closely with the predicted optimal distributions of grammatical values (see~\cref{fig:optimal_discrim} and \cref{fig:token_distribution}), Arabic, Dutch and Swedish show notable deviations, approaching the patterns of the random baseline in the animate set. These languages are currently undergoing grammatical simplification, involving different stages of neutralization of [Masc] and [Fem] into a common value encoding \textit{+animate}~\cite{fehri2020number, van2009emergence, brouwer2017processing}. 
The deviation of value distributions from predicted optima is likely a hallmark of the ongoing reorganization, potentially leading the language toward a new optimal state that maintains discriminability while using fewer values. 
This suggests that while some deviation from optimality is tolerable, more substantial alterations in distribution entail grammatical change. Future longitudinal measurements could identify the thresholds at which stability breaks down, as well as the processes driving these shifts, offering deeper insight into the dynamics of language change beyond the static patterns observed in our current data.

\section*{Conclusion}
Our results show that grammatical systems arise from a trade-off between encoding semantic content and minimizing processing cost. Languages compress a rich set of semantic attributes into a smaller set of grammatical values, namely a grammatical system. The maximum entropy of this set defines an upper bound on the information that the grammatical system can encode. 
In practice, the maximum bound is rarely reached in the actual use of languages; in other words, meaning encoding is never maximized in grammatical values. Communicative need modulates encoding at the level of individual communicative instances: when need is low, speakers use less precise and less costly forms, corresponding to underspecified values. 
At the same time, the upper bound resulting from the encoding of semantic attributes also constrains the maximum entropy available for functional purposes, explaining why in no language the set of formal values exceeds the one of semantically interpreted values. 
Overall, grammar inherits its distinctions from semantics and exploits them to sustain efficient processing rather than maximal meaning transmission. Grammars repurpose concrete, perception-based meanings to build abstract functional structures that can be used to express anything.

\begin{figure*}[t]
\centering
\includegraphics[width =\linewidth]{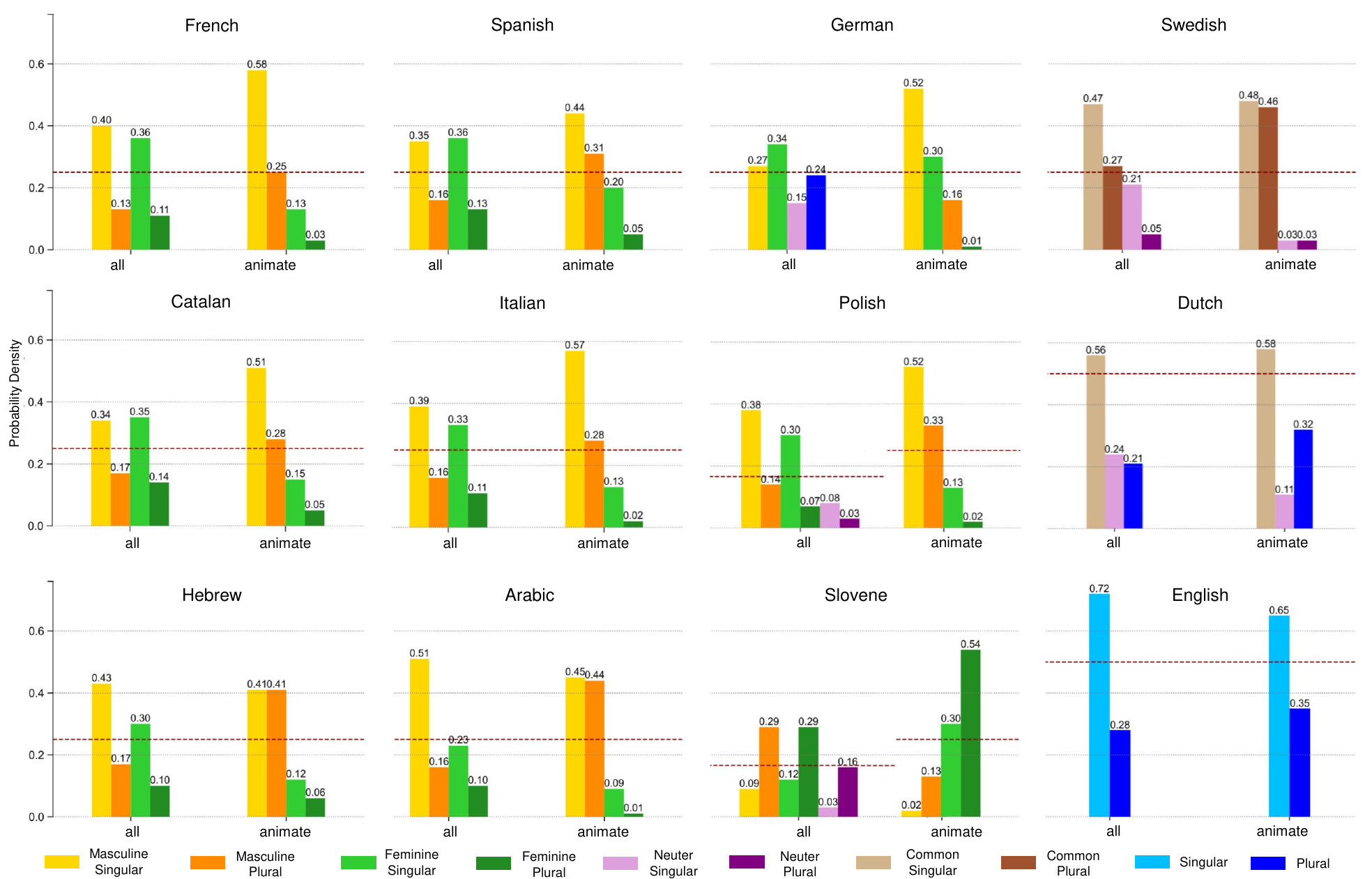} 
\caption{The grammatical value distributions for gender $\times$ number are plotted for the \emph{all nouns} set and \emph{animate subset}, for all 12 languages studied: Romance (French, Spanish, Catalan, Italian), Semitic (Hebrew, Modern Standard Arabic), Germanic (German, Swedish, Dutch, English), and Slavic (Polish, Slovene). The animate subset is defined for all languages except the fourth column, as a set of nouns which inflect for the masculine and feminine, encoding biological sex or gender. For such languages, the entropy of the gender value is positive in animates (greater than that of an inanimate referent, which has 0 entropy). The reference uniform distribution is plotted as a dashed horizontal red line. The dual number, present in Semitic languages and Slovene, is not shown as it contributes little probability mass.}
\label{fig:token_distribution}
\end{figure*}

\subsection*{Material and methods}
Our study relies on the information-theoretic framework in Shannon's theory of communication \citep{shannon1948mathematical}. Here, we define and motivate the information-theoretic terms in \eqref{eq:objective}, describe our analysis of grammatical distributions in 12 languages, and briefly discuss numerical simulations of \eqref{eq:objective} (elaborated in the Appendix).

\subsubsection*{Information-theoretic formalism} Entropy and other information-theoretic measures provide the mathematical formalism for Shannon's theory of communication \citep{shannon1948mathematical}. As such, it is natural to describe the communicative pressures in \eqref{eq:objective} such as memory, surprisal, and semantic consistency in information-theoretic terms \citep{Futrell_Hahn_2022,gibson2019efficiency}. Here, we define key terms \emph{entropy}, \emph{conditional entropy}, and \emph{informativity} (mutual information).

In \eqref{eq:objective}, we use the Shannon entropy $\mathcal H(\cdot)$ to describe memory cost at the instance level and discriminability at the system level. The Shannon entropy $\mathcal H(X)$ of a discrete random variable $X$ supported on $\mathcal X$ is defined as 
\begin{equation}
    \mathcal H(X) \triangleq - \sum_{x \in \mathcal X} p(x) \log p(x),
\end{equation}
and describes the amount of \emph{information} in bits contained in $X$. Entropy is moreover an intuitive measure of \emph{cost}, where cost is equivalent to \emph{coding length} of $X$. As such, in \eqref{eq:objective} we describe the memory complexity given a referent $\mathcal H(W_A|X=x)$ as the feature's coding length. Entropy describes the information content of $X$, but, equivalently, its uncertainty or dispersion. At the system level, discriminability is optimal when grammatical values are maximally disperse in the lexicon \citep{franzon2023entropy}, which motivates the max-entropy term $\mathcal H(W_A)$ in \eqref{eq:objective}.

 One can also measure the average surprisal, or \emph{conditional entropy}, of a random variable $Y$ conditioned on random variable $X$. This is the coding length of $Y$ given $X$ and is given by
\begin{equation}
    \mathcal H(Y|X) \triangleq - \sum_{x\in \mathcal X, y \in \mathcal Y} p(x,y) \log\frac{p(x,y)}{p(y)},
\end{equation}
where $Y$ is supported on the set $\mathcal Y$. The higher the conditional entropy $\mathcal H(Y|X)$, the harder it is to predict $Y$ given $X$. As such, $W_A$ is more informative about $A$ given $X=x$ if the coding length of $A$ decreases upon knowing $W_A$. On the system level, conditional entropy $\mathcal H(W_A|A)$ formalizes \emph{consistency} of grammatical value given semantic value: $W_A$ should be predictable given $A$. 

Lastly, \emph{informativity}, or \emph{mutual information} between two random variables $X$ and $Y$, is given by
\begin{equation}
    I(X;Y) \triangleq \mathcal H(Y) - \mathcal H(Y\mid X) = \mathcal H(X) - \mathcal H(X \mid Y).
\end{equation}
Mathematically, $I(X;Y) \geq 0$ \citep{cover_thomas}. When $I(X;Y) > 0$, then learning about $X$ reduces surprisal about $Y$ on average, and vice-versa. We use informativity in \Cref{def:inheritance} and \Cref{thm:inheritance} to describe semantic inheritance; if a grammatical feature is informative about some semantic feature for some set of referents, then we say that the grammatical features inherits from that semantic feature.

\subsubsection*{Real language data}
We analyzed the distribution of nouns across web-based large-scale corpora in 12 languages~\cite{wikidump2}. Details on the dataset and its processing are reported in the Appendix.
We sampled languages from different groups, representing different grammatical systems: German, English, Dutch, Swedish (Germanic); French, Italian, Spanish, Catalan (Romance); Polish, Slovene (Slavic): Hebrew and Arabic (Semitic). We obtained grammatical gender and number information for noun occurrences using Stanza, a pre-trained, neural part-of-speech and grammatical features tagger \cite{qi-etal-2020-stanza}.  We refer to the above data as the set of \emph{all nouns}. 

Except for English, Dutch, and Swedish, all languages tested have a set of animate nouns where grammatical gender is semantically interpreted, often referring to animals, nationality, or profession (e.g., `una amiga', `un amigo' in Spanish). For those 9 languages, we created the subset of animate nouns by filtering the set of all nouns for animacy, then choosing those inflected or derived in both the masculine and feminine. Due to noise in the automatic tagger, the animate subset was checked by hand for each language. 

In both all nouns and the animate subset, we estimate the entropy of grammatical values $W_A$. We also report their Kullback-Leibler divergence ($\mathcal D_{KL}$) to the uniform distribution, where $\mathcal D_{KL}$ quantifies the disparity between the observed distribution $p$ and the optimal uniform distribution $q$. The $\mathcal D_{KL}$ is defined as follows:

\begin{equation}
    \mathcal D_{KL}(q||p) \triangleq \sum_{x \in \mathcal X} q(x) \log \frac{q(x)}{p(x)}.
\end{equation}

$\mathcal D_{KL}(q||p)$ describes the excess coding length when using the uniform distribution $q$ to describe the observed $p$. Lower $\mathcal D_{KL}$ from $p$ to the uniform distribution means $p$ is closer to uniform, and in our case, more optimal for agreement-based discriminability.

\subsubsection*{Numerical simulations}
For toy languages of $k\in \{2,4,6,16,32\}$ words, we optimized the distribution of grammatical values for a single feature with respect to \eqref{eq:objective}, systematically varying $\alpha_x$ at the instance level, then $\beta$ at the system level. As \eqref{eq:objective} is a complicated constrained optimization problem, guaranteeing convergence was nontrivial: we reformulated \eqref{eq:objective} as an unconstrained problem that was differentiable with respect to the encoder $f$, then performed gradient-based optimization of $f$. For mathematical and implementation details, see \emph{Validating analytic results with simulated languages} in the Appendix.

\subsection*{Data and Code Availability}
Language corpora were obtained from the Huggingface Wikipedia corpus \cite{wikidump2}. The code used to perform the analysis and modeling is available at \url{https://osf.io/fmsx7/?view_only=34807651d87d465dba2ecd3046b9f43e}.

\vfill

\pagebreak
\paragraph{Acknowledgments} This project has received funding from the European Research Council (ERC) under the European Union’s Horizon 2020 research and innovation programme (grant agreement No. 101019291). This paper reflects the authors’ view only, and the funding agency is not responsible for any use that may be made of the information it contains. The authors also received financial support from the Catalan government (AGAUR grant SGR 2021 00470).

We are grateful to Claudia Artiaco for useful feedback on the manuscript, and to Marco Baroni, Gemma Boleda, Jeanne Bruneau-Bongard, and Lucas Weber, Louise McNally, the members of the COLT group at the Universitat Pompeu Fabra and the Comparative Linguistics group at U.~Zurich for valuable feedback and discussion.

\end{multicols}

\bibliography{matching_bibliography}

\pagebreak
\subsection*{Appendix: Principles of semantic and functional efficiency in grammatical patterning}

\appendix
\renewcommand\theequation{\arabic{equation}}
\renewcommand\thelemma{\arabic{lemma}}
\renewcommand\thetheorem{\arabic{theorem}}
\renewcommand\thecorollary{\arabic{corollary}}

\subsection*{Extension: Deriving universals of grammatical systems}
\setcounter{lemma}{0} 
\setcounter{theorem}{0} 
\setcounter{corollary}{0}

We restate the grammatical encoding objective and its detailed mathematical derivation, as well as results and proofs for the theoretical section of the main paper.

\subsection*{Objective derivation}

We restate the objective below in \eqref{eq:objective_app}. This is a multi-level optimization, where the inner optimization encodes the instance-level objective and the outer optimization encodes the system-level objective. This can be seen as cutting out regions of the candidate space $\mathcal F$ of all possible encoders in two steps. Optimizing the instance-level objective first restricts $\mathcal F$ to $\mathcal G \subset \mathcal F$. Then, optimizing the system-level objective searches within $\mathcal G$ to produce the final set of optimal solutions $f^*$. 

\begin{tcolorbox}[title=Grammatical organization objective, colframe=blue!60!black, colback=blue!10,left=0mm, top=0mm]
\begin{subequations}\label{eq:objective_app}
\begin{align}
\underset{f}{\text{max}} & \qquad \underset{\text{discriminability}}{\mathcal H(W_A)} & \text{\textcolor{gray}{\textbf{system level}: functional processing}}\label{eq:system_level}\\
\text{s.t.} 
& \qquad f \in \underset{{f\in \mathcal G}}{\text{argmin}} \left \{\underset{\text{size}}{|W_A|} + \underset{\text{consistency}}{\beta \mathcal H (W_A \mid A)} \right \} & \text{\textcolor{gray}{simplicity}} \label{eq:consistency} \\
\text{where} \nonumber \\
\mathcal G = & \underset{x \in \mathcal X}{\bigcap} \underset{f \in \mathcal F}{\text{argmin}} \left\{ \underset{\text{memory}}{\mathcal H(W_A\mid X\text{=}x)}+\alpha_x \underset{\text{surprisal}}{\mathcal H(A\mid W_A,X\text{=}x)} \right\} & \text{\textcolor{gray}{\textbf{instance level:} semantic encoding, simplicity}}\label{eqn:word_level} \\
& \qquad f: A \mapsto W_A \label{eqn:f_definition}\\
&  \qquad \beta \geq 0 \label{eqn:beta} \\
& \qquad \alpha_x \geq 0 \ \ \forall x \in \mathcal X \label{eqn:alpha} 
    \end{align}
\end{subequations}
\end{tcolorbox}

\noindent 
In what follows, we separately state and derive the instance-level and system-level problems.

\subsubsection*{Instance-level objective} We instantiate Zipf's Principle of Least Effort as a constrained optimization problem, see \eqref{eq:primal_word}. The goal is, for each referent $x$, to minimize the memory of storing and transmitting the message $W_A$ given $X=x$, while requiring surprisal to be below a threshold $c_x \geq 0$. The memory minimization enters as objective \ref{eqn:word_level_obj} while surprisal enters as constraint \ref{eq:word_level_constr}.

\begin{tcolorbox}[title=Instance-level objective, colframe=orange!60!black, colback=orange!10,left=0mm, top=0mm]
\begin{multicols}{2}
\textbf{Primal (constrained)} 

\begin{subequations}\label{eq:primal_word}
\begin{align}
\underset{f \in \mathcal F}{\text{min}} & \qquad \underset{\text{memory}}{\mathcal H(W_A\mid X\text{=}x)}  \qquad \forall x \in \mathcal X \label{eqn:word_level_obj} \\
\text{s.t.} 
& \qquad \underset{\text{surprisal}}{\mathcal H(A\mid W_A,X\text{=}x)} \leq c_x \label{eq:word_level_constr} \\
\text{where} \nonumber \\
& \qquad f: A \mapsto W_A \\
& \qquad c_x \geq 0 \ \ \forall x \in \mathcal X 
\end{align}
\end{subequations}
\columnbreak

\textbf{Dual (unconstrained)} \vspace{1ex}
\begin{subequations}\label{eq:dual_word}
\begin{align}
\underset{f \in \mathcal F}{\text{min}} & \qquad \underset{\text{memory}}{\mathcal H(W_A\mid A,X\text{=}x)} + \alpha_x \underset{\text{surprisal}}{\mathcal H(A\mid W_A,X\text{=}x)} \label{eqn:word_level_dual_obj} \\
& \qquad \forall x \in \mathcal X \nonumber \\
\text{where} \nonumber \\
& \qquad f: A \mapsto W_A \\
& \qquad \alpha_x \geq 0 \ \ \forall x \in \mathcal X 
\end{align}
\end{subequations}
\end{multicols}
\end{tcolorbox}

It is possible to formulate the constrained optimization problem (Eq.~\ref{eq:primal_word}) as an unconstrained one (Eq.~\ref{eq:dual_word}). This is a classic reformulation in the constrained optimization literature, where $\alpha_x$ is a function of $c_x$ \cite{bertsekas1966}. We write both equivalent formulations above, where the dual formulation is the one that figures in the main paper.

\subsubsection*{System-level objective} Now, we move from the instance level to the system level. Assume that $f \in \mathcal G$, the set of optimal solutions in the instance-level objective \eqref{eq:word_level_constr}. Similar to the instance-level objective, the system-level objective is motivated by Zipf's PLE. At the system level, simplicity is first constrained (Eq.~\ref{eq:consistency}), which sets the number of possible values of $W_A$. Finally, the max-entropy term in \eqref{eq:system_level} is responsible for assigning referents to vocabulary items in a maximally dispersed way.

We first consider the simplicity constraint in \eqref{eq:consistency}. A closer look shows that the simplicity constraint is, itself, a constrained optimization problem. The equivalent primal problem to \eqref{eq:consistency} is written in \eqref{eq:primal_simplicity} below on the left, with the unconstrained formulation rewritten on the right hand side:

\begin{tcolorbox}[title=System-level simplicity constraint, colframe=orange!60!black, colback=orange!10,left=0mm, top=0mm]
\begin{multicols}{2}
\textbf{Primal (constrained)} 
\begin{subequations}\label{eq:primal_simplicity}
\begin{align}
\underset{f \in \mathcal G}{\text{min}} & \qquad \underset{\text{size}}{|W_A|} \label{eqn:simplicity_obj} \\
\text{s.t.}& \qquad \underset{\text{consistency}}{\mathcal H(W_A\mid A) < c} \\
\text{where} \nonumber \\
& \qquad f: A \mapsto W_A \\
& \qquad c > 0
\end{align}
\end{subequations}
\columnbreak

\textbf{Dual (unconstrained)}
\vspace{1ex}
\begin{subequations}\label{eq:dual_system}
\begin{align}
\underset{f \in \mathcal G}{\text{min}} & \qquad \underset{\text{size}}{|W_A|} + \beta \underset{\text{consistency}}{\mathcal H(W_A\mid A)} \label{eqn:simplicity_dual_obj} \\
\text{where} \nonumber \\
& \qquad f: A \mapsto W_A \\
& \qquad \beta \geq 0
\end{align}
\end{subequations}
\vfill
\end{multicols}
\end{tcolorbox}

The pressure towards simplicity at the system level further constrains the set of $f$, by imposing global correlations between the attribute-value mappings for referent. This effectively constrains the overall vocabulary space $\mathcal W_A$; where without it, the plural for `cat' may differ from the plural for `friend', the consistency term forces them to be the same. Likewise, the size reduction term adds a general pressure towards sparsity, and eliminates synonyms within and between referents. After these simplicity conditions are satisfied, entropy maximization is applied (Eq.~\ref{eq:system_level}) to assign nouns to grammatical values in a maximally dispersed way. This completes the system-level objective.

\subsection*{Proof: \cref{thm:inheritance}}
Here, we state and prove \cref{thm:inheritance}.

\begin{theorem}[Grammatical values inherit from semantics]
\label{thm:inheritance}
    If $f$ optimizes \eqref{eq:objective_app}, then each $W_A$ either inherits from some $A$ or consists of a single neutralized value for the whole lexicon.
\end{theorem}
\begin{proof}
    Consider a given $x$ and $A$. If $\alpha_x = 0$, then $f$ sends every value in $A$ to the same value, $f_\alpha(x)=\emptyset_x$. If $\alpha_x > 0$, then $H(A\mid W_A,X=x) < H(A\mid X=x)$ ($W_A$ is informative about $A$ for $X=x$). 

    Let $\mathcal X_A = \{x; \ \alpha_x>0, x\in \mathcal X\}$ be the set of referents $x$ for which $\alpha_x > 0$. First, assume $\mathcal X_A$ is non-empty. Then,
    \begin{align}
        H(A\mid W_A, x\in\mathcal X_A) &= \sum_{x \in \mathcal X_A} p(X=x\mid x\in\mathcal X_A)H(A\mid W_A,X=x) \\
        &< \sum_{x \in \mathcal X_A} p(X=x\mid x\in\mathcal X_A)H(A\mid X=x) \\
        & = H(A\mid x \in \mathcal X_A).
    \end{align}
    Hence, for each $A$ where $\mathcal X_A$ is non-empty, there is a corresponding $W_A$ supported over at least two values.

    Now, consider $\mathcal X_A$ empty. This means that $\alpha_x=0$ for all $x \in \mathcal X$. \Cref{eq:consistency} then enforces that $\emptyset_x = \emptyset_{x'} = \emptyset$ for all $x\neq x' \in \mathcal X$, or that $W_A$ produces the same neutralized null value, or empty string, for the entire lexicon. 
\end{proof}

\subsection*{Proofs: \cref{lemma:num_feats} and \cref{thm:num_feats}}
Here, we state the proofs of \cref{lemma:num_feats} and \cref{thm:num_feats}. To do so, we apply the constraint at the instance level \eqref{eqn:word_level}, and then apply the simplicity constraint at the system level \eqref{eq:consistency}. 

Our proofs rely on the stipulation that, at each step of the multi-level optimization, the feasible sets (for instance, $\mathcal G$) are non-empty. $\mathcal G$ is indeed not empty. To see why this is the case, note that $f$ can be represented as a vector of size $|\mathcal X|$, where each independent entry $f_x$ is indexed by a referent $x \in \mathcal X$. It suffices to show that a globally optimal $f_x$ exists for all $x$ under any choice of $\alpha_x \geq 0$; this is true because \Cref{eq:word_level_constr} constrains the set of possible $f_x$ to be closed and bounded.

\begin{lemma}[Number of grammatical values given referent]
\label{lemma:num_feats}
    If $f^* \in \mathcal G$ \eqref{eqn:word_level}, it satisfies $|W_{A\mid X=x}| \leq |A_{X=x}|$ $\forall x \in \mathcal X$.
\end{lemma}

\noindent We first prove a weaker claim, then relax the condition to prove the lemma.

\begin{claim}\label{claim:hard_constraint}
    Let $f^*$ minimize $\mathcal H(W_A \mid  X\text{=}x)$ subject to minimizing surprisal $\mathcal H(A\mid  W_A, X\text{=}x)$. Then, the number of values of $W_{A\mid X\text{=}x} = f^*(A\mid X\text{=}x)$ is the number of possible values of semantic attribute $A$:
    $$|W_{A\mid X=x}| = |A_{X=x}|.$$
\end{claim}
\begin{proof}
    Fix $x$. We first show that $| W_{A\mid X=x}| \geq |A_{X=x}|$. When $f$ minimizes $\mathcal H(A \mid  W_A, X\text{=}x)$, the reverse mapping
    $$g: W_{A\mid X=x} \mapsto A_{X=x}$$
    must be a function. This constrains $| W_{A\mid X=x}| \geq |A_{X=x}|$. Let $W_{A\mid X=x}$ be supported on $\mathcal W_{A\mid X=x}$ and $A_{X=x}$ be supported on $\mathcal A_{X=x}$. If $g$ is a function, then $f: A_{X=x} \mapsto W_{A\mid X=x}$ is a surjective partition of $$\mathcal W_{A\mid X=x} = \bigcup_{a \in \mathcal A_{X=x}} \text{im}_f(a).$$
    If surprisal $\mathcal H(A\mid W_A,X=x)$ is minimized, then the sets $\text{im}_f(a_i)$ and $\text{im}_f(a_j)$ are disjoint for all $a_i \neq a_j \in \mathcal A_{X=x}$.
    Then, the number of elements of $W_{A\mid X=x}$ is $$|W_{A\mid X=x}| = \sum_{a\in \mathcal A_{X=x}} |\text{im}_f(a)|.$$ 
    
    Now, we minimize $\mathcal H(W_A \mid  X\text{=}x)$ given minimal $\mathcal H(A\mid W_A,X\text{=}x)$. This is achieved when $|\text{im}_f(a)| = 1 \ \forall a \in \mathcal A_{X=x}$, or when $f$ is a bijection. Then, $|W_{A\mid X=x}| = |A_{X=x}|$.
\end{proof}

\noindent 
Now, to recover the constrained optimization problem in \eqref{eqn:word_level}, we relax the constraint that surprisal must be minimized. 

\begin{proof}[Proof of \cref{lemma:num_feats}]
    When strict surprisal minimization is relaxed, surprisal enters as a constraint whose threshold is modulated by $\alpha_x$ (Eq.~\ref{eqn:word_level}). Then, the mapping $g: W_{X=x} \mapsto A_{X=x}$ is no longer a function; it can be one-to-many. This relaxes the equality in \eqref{claim:hard_constraint} such that $|W_{A\mid X=x}| \leq |A_{X=x}|$.
\end{proof}

Now, we apply the system-level simplicity constraint in \eqref{eq:consistency}. The result is that the encoder $f$ co-opts the same grammatical values for different referents. Crucially, it constrains the number of grammatical values to fewer than the number of semantic attributes.

\begin{theorem}[Number of grammatical values for a feature in system]
\label{thm:num_feats}
    If $f^*$ satisfies \eqref{eq:objective_app}, then the number of values of $W_A = f^*(A)$ is less than the number of values of $A$:
    $$|W_A| \leq |A|.$$
\end{theorem}
\begin{proof}
    We start by considering $f \in \mathcal G$. By definition in \eqref{eq:objective_app}, $f \in \mathcal G$ is optimal with respect to the instance-level constraint. Then, by \eqref{lemma:num_feats}, $|W_{A,X\text{=}x}| \leq |A_{X=x}|$ for each $x$. 

    Now, we apply the system-level constraint. We first prove a stronger condition, where size is minimized subject to maximizing consistency ($\beta \to \infty$). Then, we use this condition to arrive at the final result. 
    
    Let $f^*$ maximize consistency, so $f^* \in \text{argmin}_{f \in \mathcal G} H(W_A\mid A)$. Note that the minimum $\mathcal H(W_A \mid A)=0$ is achieved by some $f^* \in \mathcal G$: this happens for all injective $f$, or that, over all $x \in X$, the same value of $A$ yields the same $W_A$. Then, $W_{A=a, X=x_i}=W_{A=a, X=x_j}$ for all $x_i, x_j \in \mathcal X$ and for each $a \in \mathcal A$. It follows that $|W_A| \leq |A|$ in the whole lexicon for such $f^*$.

    Now, starting from $f^*$, we construct a new solution $f^{**}$ so that the original \eqref{eq:primal_simplicity} is satisfied. Note that, here, it is valid to analyze the unconstrained formulation \eqref{eqn:simplicity_dual_obj} using the primal constrained one \eqref{eq:primal_simplicity} for two reasons: (1) they can be made equivalent by choice of $c$ and $\beta$, and (2) the feasible set $\{f: \mathcal H(W_A\mid A)< c\}$, $c > 0$, which includes $f^*$ where $H(W_A\mid A)=0$, is non-empty. Recall that, by definition, $f^*$  achieves $\inf_{f \in G} \mathcal H(W_A\mid A) = 0 < c$, satisfying the constraint. Let the size of the set of grammatical values produced by $f^*$ be $|W_A|^*$. We showed that $f^*$ satisfies $|W_A|^* \leq |A|$. Then, the minimal $|W_A|$ solution $f^{**}$ to \eqref{eq:primal_simplicity} must have $|W_A|^{**} \leq |W_A|^{*} \leq |A|$. 
\end{proof}

\noindent 
We have shown that, if a solution $f^*$ exists in \eqref{eq:objective_app}, it must satisfy $|W_A| \leq |A|$. The implication for grammars is that languages optimizing \eqref{eq:objective_app} have grammatical values constrained by the number of semantic attributes. Then, since we model the semantic features as orthogonal to each other, it is easy to bound the total number of grammatical values by the total number of semantic attributes:

\begin{corollary}[Total features in a system]
    For $k$ orthogonal semantic attributes $\mathbf{A} \triangleq \{A_i\}_{i=1}^k$ and corresponding grammatical features $\mathbf{W} \triangleq \{W_{A_i}\}_{i=1}^k$, the total number of values 
    $$|\mathbf{W}| \leq |\mathbf{A}|.$$
\end{corollary}
\begin{proof}
    From \eqref{thm:num_feats}, $|W_{A_i}| \leq |A_i|$ $\forall 1 \leq i \leq k$. Then,
    $$|\mathbf{W}| = \prod_{i=1}^k |W_{A_i}| \leq \prod_{i=1}^k |A_i| = |\mathbf A|,$$
    where the left and right equalities are due to orthogonality of attributes and features.
\end{proof}

\subsection*{Proof of \cref{lemma:entropy_feat_ref}}
Now, we state and prove \cref{lemma:entropy_feat_ref}.

\begin{lemma}[Entropy of grammatical values given referent]
\label{lemma:entropy_feat_ref}
    The entropy of $W_A = f^*(A)$, where $f^*$ satisfies \eqref{eq:objective_app}, is bounded above by the entropy of semantic feature $A$:
    $$\mathcal H(W_A\mid X\text{=}x) \leq \mathcal H(A \mid  X=x)$$
    The equality condition occurs when surprisal $H(A\mid W_A,X\text{=}x)$ is minimized.
\end{lemma}
\begin{proof}
    As in the proof of \cref{thm:num_feats}, we find a solution $f^*$ to a stronger condition, then use $f^*$ to show that the present lemma must hold in the general case.

    The present lemma concerns the instance-level constraint \eqref{eqn:word_level}. First, let $f^*$ satisfy a stronger condition: it minimizes $\mathcal H(W_A\mid X\text{=}x)$ subject to minimizing surprisal $\mathcal H(A\mid W_A,X\text{=}x)$. The solutions that minimize surprisal are those that map each $A=a$ to a disjoint set of values $W_{A=a, X=x}$. Then, applying entropy minimization, the optimal $f^*$ is a bijection between $A_{X=x}$ and $W_{A, X=x}$. This means that an optimal $f^*$ with respect to this stronger condition has $\mathcal H(W_A\mid X=x) = \mathcal H(A\mid X=x)$. The equality condition of the present lemma is achieved by $f^*$.

    Now, we use $f^*$ to reason about the set of solutions $\mathcal G$ to the actual instance-level constraint \eqref{eq:primal_word}. Note that $f^*$ already satisfies the surprisal constraint \eqref{eq:word_level_constr}, as it achieves the global minimum. Let the conditional entropy $\mathcal H(W_A\mid X=x)$ of $f^*$ be $\mathcal H^*$. Then, any solution $f^{**}$ that is more optimal than $f^*$ with respect to \eqref{eq:primal_word} has a lower entropy $\mathcal H^{**}(W_A\mid X\text{=}x) \leq \mathcal H^* = \mathcal H(A\mid X=x)$. 
\end{proof}

\newpage 
\section*{Data Collection and Preprocessing}
Here, we outline the preprocessing steps to obtain the set of all nouns and the animate subset for all languages. All analyses are based on the language-specific Wikipedias. Although we considered other data sources such as Universal Dependencies \citep{nivre-etal-2020-universal} and Wacky Wide Web corpora \citep{inproceedings_ukwac,article_wac}, we chose Wikipedia due to a combination of its large size and typological spread. For a single language, the steps in our pipeline are summarized as follows:

\begin{enumerate}
    \item Tag entire Wikipedia corpus for part-of-speech, lemma, grammatical gender, and grammatical number
    \item To form the set of all nouns, take the intersection of words tagged as \texttt{NOUN} by Stanza and by WordNet.
    \item To form the set of animate nouns, take the lemmas that are in the set of words labeled \texttt{PERSON} or \texttt{ANIMAL} in WordNet.
\end{enumerate}

All preprocessing steps were done in Python, using publicly available resources.

\subsection*{Wikipedia data preprocessing}
The Wikipedia data are sourced from Huggingface, found at \url{https://huggingface.co/datasets/wikimedia/wikipedia}. Data consists of entire articles, cleaned to strip markdown and references, with a time cutoff of January 20, 2023 for all languages.

Using Stanza \citep{qi-etal-2020-stanza}, an automatic tagger, we convert each article to CONLL-U format \citep{buchholz-marsi-2006-conll}, which lists each word tagged with its part-of-speech (POS), lemma, and morphological features (grammatical features like gender and number). The available POS and morphological features for each language may be found on the language-specific reference treebanks listed on the Stanza webpage: \url{https://stanfordnlp.github.io/stanza/performance.html}.

\subsection*{Selecting the set of all nouns}
The results of the automatic tagger were often noisy, such that words tagged \texttt{NOUN} may have been, for instance, proper nouns or non-words. To remedy this, we culled the set of nouns by restricting the set of words tagged \texttt{NOUN} by Stanza to those also tagged \texttt{NOUN} by WordNet, a lexicographic resource available for many languages in the \texttt{nltk} Python package \citep{bird-loper-2004-nltk}. This yields the set of all nouns for all the languages. 

\subsection*{Selecting the animate subset}
We then selected a subset of animate nouns for which grammatical gender is semantically interpreted for Romance, Slavic, Semitic languages, and German (group A). For completeness, we did the same for Dutch, Swedish, and English (group B), although the gender and number distributions on the animate subset were not relevant for analysis. 

First, for both groups A and B, we selected a subset of nouns from the all-nouns set that were labelled \texttt{ANIMAL} or \texttt{PERSON} in WordNet, that is, nouns that have a semantic gender or biological sex. The German WordNet \citep{siegel-bond-2021-odenet} was sparsely annotated for these domains, so we translated the lemmas from English WordNet \citep{englishwordnet} to German, allowing for better coverage. We stopped here for group A. For group B, we chose the animate subset as follows. We kept all lemmas that exhibit all combinations of $\{$Fem, Masc$\} \times \{$Sing, Plur, (Dual if applicable)$\}$. This was considering both inflection and derivation, where inflection would be, for example, \emph{gato} [+Masc] / \emph{gata} [+Fem] (Spanish, cat), and derivation modifies the ending of the word as in \emph{der Student} [+Masc] / \emph{die Studentin} [+Fem] (German, student). This creates a subset of nouns where types are distributed uniformly across all possible combinations, while tokens are not. The token distribution, which captures language \emph{use}, is then our distribution of interest. 

The final counts for the set of all nouns and set of animate nouns, on both types and tokens, is shown in \cref{tab:types_tokens}.

\begin{table}[]
    \centering
    \caption{Number of tokens (noun occurrences) and types (unique nouns) in each dataset used for analysis. Animate tokens and types were compiled by restraining lemmas to those tagged \texttt{person} or \texttt{animal} in WordNet. For English, Dutch, and Swedish, no further preprocessing was done, hence the high counts for these languages. For remaining languages, we took all animate nouns for which the grammatical gender expresses semantic gender, took a 50/50 subset of masculine and feminine in the singular, and then included the plural versions.}
    \begin{tabular}{ccccc}
    \toprule 
       Language & Tokens & Types & Animate Tokens & Animate Types \\
       \midrule 
       Catalan & 46196329 & 29732 & 2126465 & 636 \\
       French & 229417715 & 40387 & 11257032 & 1361 \\
       Italian & 120227703 & 31566 & 6280036 & 1126 \\
       Spanish & 120884503 & 21503 & 3152594 & 348 \\
        \hline
       Polish & 76467551 & 38987 & 549157 & 271 \\
       Slovene & 12983495 & 23433 & 302300 & 454 \\
       \hline 
       Arabic &	8455334 & 1131 & 1380575 & 708 \\
       Hebrew & 25532170 & 5158 & 1813007 & 365 \\ 
       \hline 
       German & 141146984 & 97612 & 1700507 & 188 \\
       English & 411741299 & 69474 & 107355013 & 11715 \\
       Dutch & 42997438 & 42397 & 7869393 & 6568 \\
       Swedish & 26785332 & 7890 & 3711673 & 3100 \\
       \hline 
    \end{tabular}
    \label{tab:types_tokens}
\end{table}
\FloatBarrier

\newpage
\section*{Comparing Empirical Results to Hypothetical Languages}
We perform hypothesis testing to show that the entropy of grammatical values (Fig.~2 in the main text) behaves closer to the uniform distribution than that of a class of random, hypothetical languages. That is, we show that the observed $\mathcal D_{KL}$ to the uniform distribution is significantly \emph{less} than that of a random language. This baseline class of random hypothetical languages is drawn from a Dirichlet distribution: for each real language in our set, see Table 2 (main), we consider $\mathcal D_{obs}$, its KL-divergence to the uniform distribution. Let $k$ be the number of grammatical values for a certain attribute $A$. Recall that $\mathcal D_{KL}=0$ means the distribution is uniform. We show that $\mathcal D_{obs}$ is lower (closer to 0) than $\mathcal D_{L}$, defined as the KL-divergence to the uniform distribution for the baseline languages $L$. Let the distribution of $\mathcal D_{L}$ be $p_{\mathcal D_L}$. Formally, we test the following hypotheses:
\begin{itemize}
    \item \textbf{Null hypothesis $H_0$.} $\mathcal D_{obs}$ is drawn from the baseline distribution $p_{D_L}$. $\mathbb P(\mathcal D_{L} \leq \mathcal D_{obs} \mid  \mathcal D_L \sim p_{\mathcal D_L}) \geq \alpha$. 
    \item \textbf{Alternate hypothesis $H_A$.} $\mathcal D_{obs}$ is \emph{not} drawn from the baseline distribution $p_{D_L}$. $\mathbb P(\mathcal D_{L} > \mathcal D_{obs} \mid  \mathcal D_L \sim p_{\mathcal D_L}) < \alpha$.
\end{itemize}

\paragraph{Distributions of $L$ and $\mathcal D_{L}$} We choose $L \sim \text{Dir}(\frac{1}{k} \mathbf{e}_k)$. Each $L$ is a categorical distribution with $k$ categories, drawn from a Dirichlet distribution with uniform concentration. We justify the choice of uniform concentration $\frac{1}{k}\mathbf e_k$ in that it does not impose a bias towards one category or another; for instance, assuming there are $k=4$ possible values, it does not impose that $p(\text{MascSing}) > p(\text{MascPlur})>p(\text{FemSing})>p(\text{FemPlur})$. 

For each number of grammatical values $k \in \{2,3,4,6,9\}$ attested in our set of real languages, see the x-axis of Fig.~2, we sample $N=1000$ baseline languages $\{L_i\}_{i=1}^N$. We obtain the empirical KL distributions $\hat p_{\mathcal D_L}^{(k)}$ by computing the KL-divergences to the uniform distribution $\{\mathcal D_{L,i}\}_{i=1}^N$ of each sampled distribution $\{L_i\}_{i=1}^N$. 

\paragraph{Significance testing} We test the null hypothesis $H_0$ using a one-sided empirical percentile-rank test, setting the p-value cutoff to $\alpha=0.05$.\footnote{We chose a non-parametric test due to the observed non-normality of $\hat p_{\mathcal D_L}$, verified for each $k$, $N=1000$, with Shapiro-Wilk tests at a $p$-value cutoff of $0.05$} That is, we estimate $\mathbb P(\mathcal D_{L} \leq \mathcal D_{obs} \mid  \mathcal D_L \sim p_{\mathcal D_L})$ from the empirical distribution of $\mathcal D_{L}$ as follows: 

\begin{equation}
\label{eqn:significance}
    \mathbb P(\mathcal D_{L} \leq \mathcal D_{obs} \mid  \mathcal D_L \sim p_{\mathcal D_L})  \approx \frac{\sum_{i=1}^N \mathbf{1}(\mathcal D_{L,i} < \mathcal D_{obs})}{N}.
\end{equation}

We apply \Cref{eqn:significance} for each language, for the overall grammatical system shown in Fig.~2, as well as for each $A$ separately in $\{\text{gender}, \text{number}\}$. If the estimated $\mathbb P(\mathcal D_{L} < \mathcal D_{obs} \mid  \mathcal D_L \sim p_{\mathcal D_L}) < \alpha=0.1$, then we reject the null hypothesis in favor of $H_A$: $\mathcal D_{obs}$ is less than chance, or the observed grammatical distribution is closer to uniform than by chance.

\begin{remark}
    Due to monotonicity, p-values derived for $H_0: \mathcal D_{L}<\mathcal D_{obs} \mid  \mathcal D_L \sim p_{D_L}$ hold for potential null hypotheses $\mathcal H_L > H_{obs}$ as well as $\text{Agr}D_L > \text{Agr}D_{obs}$, where baseline languages $L$ are sampled from the same distribution $\text{Dir}(\frac{1}{k}\mathbf e_k)$.
\end{remark}

\paragraph{Results: overall lexicon} We first consider the entire lexicon, or the set of all nouns (Table 2, main text). For all languages but English, the observed $\mathcal D_{KL}$ for the \textbf{overall system} is significantly lower than that of a random categorical distribution at a p-value cutoff of $\alpha=0.1$; for all remaining languages but Dutch and Semitic languages, the result is significant at $\alpha=0.05$. This means that for all languages but English, the grammatical system's distribution is significantly closer to being uniform than by chance. Our model correctly predicts English to not follow this picture. English nouns only inflect for grammatical number, which is semantically interpreted for the vast majority of nouns ($\alpha_x \to \infty$); constrained to encode semantics at the instance level, grammatical values cannot disperse at the system level. 

\paragraph{Results: gender} We performed the same hypothesis testing restricted to $A=$\textbf{gender}, see Table 2, left section. In the entire lexicon, for all languages but English (not applicable), Arabic ($p=0.21$), Dutch ($p=0.28$), and Swedish ($p=0.33$), the grammatical gender distribution is significantly closer to uniform than by chance ($p<0.1$), and highly significant ($p<0.05$) in French, Spanish, Catalan, German, and Slavic languages. Meanwhile, as shown in the left section of \Cref{tab:anim_entropies}, for the animate subset, where for German, Romance, Slavic, and Semitic languages grammatical gender encodes semantics of the referent (e.g., ``un amig\textbf{o}" [+Masc], ``una amig\textbf{a}" [+Fem] in Spanish), grammatical gender is not significantly uniform, compared to the baseline class (p-value cutoff of $0.1$). Here, $\mathcal D_{KL}$ values for the animate subset are higher than in the set of all nouns. Overall, results show that when the same grammatical feature encodes semantics (animate subset), it is less uniform than when it does not encode semantics (all nouns).

\paragraph{Results: number} For all languages, the number grammatical feature encodes numerosity in the entire lexicon. As such, our model does \emph{not} predict its values to uniformly disperse across referents. Indeed, for the set of all nouns (Table 2, middle section), $\mathcal D_{KL}$ is not significantly more uniform than the baseline class of languages (p-value cutoff of $0.1$). The interpretation is that number is tasked to fulfill the per-referent semantic encoding objective, anchoring its statistics to the distribution of referent numerosities.

\newpage
\section*{Validating Analytic Results with Simulated Languages}
Here, we empirically validate each theoretical result in the main paper section \emph{Deriving semantic universals of grammatical systems}. At the instance level, this concerns \Cref{lemma:num_feats} and \Cref{lemma:entropy_feat_ref}, and at the system level, \Cref{thm:inheritance} and \Cref{thm:num_feats}. As part of our analysis, we ablate each component in \Cref{eq:objective_app} by varying $\alpha$ and $\beta$ in \Cref{eq:objective_app} (i.e., setting them to 0) in order to simulate its influence on the structure of the lexicon, restricting our analyses to grammatical gender and number. We first consider individual referents in the instance-level objective. Then, considering a system composed of a \emph{set of referents} that already satisfy the instance-level objective, we apply the global objective. Our simulations require us to assume a distribution over referents and referent attributes, that is, we need to fix $p_X$ and $p_{A\mid X=x}$. Directly simulating natural languages would require, for instance, historical statistics on perceptual categories for numerosity, animacy, gender, or biological sex of referents in-the-wild, which is not possible. Instead, we make informed assumptions about these categories when conducting our simulations, and explore what happens outside these assumptions. 

\subsection*{Methods}
We describe the simulations for the instance level and then the system level. At the instance level, separately for $A$=gender and numerosity, we fix an $X=x$ and optimize a hypothetical grammatical system according to \Cref{eqn:word_level}. At the system level, separately for gender and numerosity, we optimize \Cref{eq:objective_app} for toy languages consisting of a collection of $k \in \{2,4,8,16,32\}$ words.

\subsubsection*{Instance level} 
At the instance level, we fix a hypothetical $x \in \mathcal X$ and fix its distribution over gender or biological sex and numerosity. Then, we simulate the optimal distribution $p_{A,W_A \mid X=x}$ by optimizing the instance-level objective \Cref{eq:primal_word} with respect to $p_{A,W_A \mid X=x}$, for various values of $\alpha_x$. These allow us to empirically validate our theoretical results at the instance level: that the entropy of grammatical features is bounded by that of semantic categories \Cref{lemma:entropy_feat_ref}, the number of features is bounded by number of semantic categories \Cref{lemma:num_feats}, and that semantic encoding strategies partition $\alpha_x$ (main, Equation 6).

In order to simulate grammatical systems for a single referent $x$, we need access to the joint distribution $p_{A,W_A\mid X=x}$. To the best of our knowledge, no existing dataset confers the true joint distribution $p_{A,W_A \mid X=x}$ for any possible $x$, for any semantic attribute $A$. Therefore, as in other studies \citep{Mollica_Bacon_Zaslavsky_Xu_Regier_Kemp_2021}, we propose plausible estimates of the marginal distributions $p_{A\mid X=x}$ for two use cases, $A=$ gender or biological sex and $A=$ numerosity. 

\paragraph{Marginal distributions} 

Fix $x \in \mathcal X$, an $\alpha_x$, and a joint distribution $p_{A,W_A\mid X=x}$ over semantic attributes $A$ and grammatical values $W_A$. Access to this joint distribution enables the estimation of memory $\mathcal H(W_A\mid  X=x)$ and surprisal $\mathcal H(A\mid W_A, X=x)$ terms. For a range of $\alpha_x \in \{0.0, 0.05, 0.1, 0.3, 0.5, 0.75, \to\infty\}$ we optimize $p_{A,W_A\mid X=x}$ with respect to the instance-level objective in \Cref{eq:dual_word}. Note that $\alpha_x=0.0$ ablates the surprisal term, while $\alpha_x \to \infty$ ablates the memory term of \Cref{eq:dual_word}.

We set some animate $X=x$, and assume its perceived biological sex distribution to be $p_{\text{Gender} \mid X=x} = [0.49, 0.49, 0.02]$ respectively for $[\text{Male}, \text{Female}, \text{Other}]$. Following the literature on the distribution of numerosities in-the-wild \citep{DEHAENE19921, dorogovtsev2006frequency, jansen2001round, Piantadosi_Cantlon_2017}, we assume $x$'s numerosity distribution follows an inverse-square law $p_{\text{Numerosity} \mid X=x}(n) \propto \frac{1}{n^2}$, where $n \in \mathbb N_{\geq 1}$. For simplicity, we cap the number of numerosity categories to $n=1\cdots 10$, which accounts for roughly $90\%$ of the entire distribution specified by the inverse-square law.

\paragraph{Optimization} For each $A=$ gender/biological sex and numerosity, we do the following. Over $N=50$ random seeds and each setting of $\alpha_x$, we optimize the instance-level objective \Cref{eq:dual_word} over the joint distribution $p_{A,W_A\mid X=x}$, where the marginal $p_{A\mid X=x}$ is fixed to be the marginal distributions described in the previous paragraph. In optimization, we use Adam \cite{2015-kingma} with default parameters $\beta=(0.9, 0.999)$, and a learning rate of $0.01$. The conditional distributions $p_{W_A\mid X=x}$ are initialized by drawing logits independently $\mathcal N(0,1)$, then normalizing to a probability distribution with the softmax function. For our simulations, we seed $|W_A|$ to be $15$-- the precise number does not matter so long as it is higher than $|A|$. We run all simulations to convergence, defined as when the relative change in \Cref{eq:primal_word} at time $t$ of optimization is less than $0.1\%$ from the previous time $t-1$.

\subsubsection*{System level} Simulations at the system level occur over a collection of words. Separately for $A=$gender/biological sex and numerosity, we optimize \Cref{eq:objective_app} over a lexicon of $k$ words $\mathcal X = \{x_i\}_{i=1}^k$, $k\in\{2,4,8,16,32\}$, where each word is associated to a joint distribution $p_{A,W_A\mid X=x_i}$ between the semantic attribute and grammatical value.

\paragraph{Marginal distributions} For both gender and number, we adopt the same marginal distributions for $A$ as for the instance level simulations. For simplicity, for each gender and number, we adopt the same marginal distribution for each referent and assume that referents are sampled uniformly. We leave a detailed investigation of the impact of referent distribution on grammar for future work.

\paragraph{Optimization} \Cref{eq:objective_app} is an example of \emph{constrained optimization}. Here, discriminability $\mathcal H(W_A)$ is optimized subject to (1) optimal simplicity at the system level $f \in \text{argmin}_{f \in \mathcal G} |W_A| + \beta \mathcal H(W_A\mid A)$; which is itself subject to (2) optimality at the instance level $\forall x \in \mathcal X$. Ensuring that the layered objectives (1) and (2) hold at each step of optimization can be difficult in practice-- notably, we found that optimization was not guaranteed to converge to the global minimum, so we report the best-of-$N$ random seeds, further explained below. Concretely, we split optimization into multiple steps:
\begin{enumerate}
    \item $\forall x_i \in \mathcal X$, perform the instance level optimization (described in the previous subsection) over $N=10$ random seeds to obtain best-of-$N$ estimates $c_i^* = \min_{f\in \mathcal F} \mathcal H(W_A\mid X=x_i) + \alpha_x \mathcal H(A \mid W_A, X=x_i)$. 
    \item During optimization, we want to ensure the instance level objective function does not deviate from $c_i^*$ for each $x_i$. We rewrite \Cref{eq:objective_app} below, proxying the global objective using unconstrained optimization as follows:
    \begin{subequations}
    \begin{align}
    \underset{f}{\text{max}} & \qquad \underset{\text{discriminability}}{\mathcal H(W_A)} \qquad  \text{\textcolor{gray}{\textbf{system level}: functional processing}}\nonumber \\
    \text{s.t.} 
    & \qquad f \in \underset{{f\in \mathcal F}}{\text{argmin}} \left \{ \overbrace{\underset{\text{size}}{|W_A|} + \underset{\text{consistency}}{\beta \mathcal H (W_A \mid A)}}^{\text{system level simplicity}} + \lambda \sum_{x_i\in \mathcal X} \overbrace{\left (\underset{\text{memory}}{\mathcal H(W_A\mid X\text{=}x_i)}+\alpha_{x_i} \underset{\text{surprisal}}{\mathcal H(A\mid W_A,X\text{=}x_i)} - c_i^* \right)^2}^{\text{instance level optimality penalty}} \right \} \label{eqn:update_1_constraint} \\
    & \qquad f: A \mapsto W_A \nonumber \\
    &  \qquad \beta \geq 0  \nonumber \\
    & \qquad \lambda > 0 \label{eqn:new_param} \\
    & \qquad \alpha_x \geq 0 \ \ \forall x \in \mathcal X \nonumber \\
    & \qquad c_i^* = \min_{f \in \mathcal F} \underbrace{\mathcal H(W_A\mid X\text{=}x_i)+\alpha_{x_i}\mathcal H(A\mid W_A,X\text{=}x_i)}_{\text{instance level objective}} \ \ \forall x_i \in \mathcal X.\nonumber 
        \end{align}
    \end{subequations}
    Here, the only difference from \Cref{eq:objective_app} is that we have incorporated in \Cref{eqn:update_1_constraint} a squared $L_2$ penalty \cite{Boyd_Vandenberghe_2004} designed to keep the instance level objective near the optimum $c_i^*$, with tradeoff $\lambda > 0$ (\eqref{eqn:new_param}). We set $\lambda$ to be constant for all referents, though this simplification can be relaxed.
    \item Over $N=50$ random seeds, we optimize \Cref{eqn:update_1_constraint} using Adam (default parameters, lr=0.01) over the collection of words $\mathbf{P} = [p_{A,W_A\mid X=x_i}]_{i=1}^k$, obtaining the best-of-$N$ estimate of the optimum $s^*$. Because $|W_A|$ is not differentiable, we replace $|W_A|$ with a smooth, differentiable proxy $g$ to the $L_0$ norm \cite{sl0}, where for any vector $\mathbf x \in \mathbb R^D$:
    \begin{equation}
        \|\mathbf x\|_0 \approx g(\mathbf x) \triangleq\sum_{i=1}^{D} 1 - \exp \frac{-x_i^2}{\epsilon},
        \end{equation}
        where
        \begin{equation}
        1 - \exp \frac{-x_i^2}{\epsilon} \overset{\epsilon \to 0}{\rightarrow}
\begin{cases}
  1 & \text{if } x_i \neq 0 \nonumber \\
  0  & \text{if } x_i = 0. \nonumber 
\end{cases}
    \end{equation} 
    \item Using the same squared penalty method as in Step 2, we introduce a penalty term for the system-level simplicity to obtain a new unconstrained objective:
    \begin{subequations}
    \begin{align}
    \min_{f \in \mathcal F}& \ \ \ -\underset{\text{discriminability}}{\mathcal H(W_A)} + \lambda_1 \overbrace{\left(\underset{\text{size}}{g(W_A)} + \underset{\text{consistency}}{\beta \mathcal H (W_A \mid A)} - s^*\right)^2}^{\text{simplicity optimality penalty}} + \lambda_2 \sum_{x_i\in \mathcal X} \overbrace{\left (\underset{\text{memory}}{\mathcal H(W_A\mid X\text{=}x_i)}+\alpha_{x_i} \underset{\text{surprisal}}{\mathcal H(A\mid W_A,X\text{=}x_i)} - c_i^* \right)^2}^{\text{instance level optimality penalty}} \label{eqn:unconstrained_global} \\
    \text{s.t.} & \qquad f: A \mapsto W_A \nonumber \\
    &  \qquad \beta \geq 0  \nonumber \\
    & \qquad \lambda_1, \lambda_2 > 0 \label{eqn:new_params} \\
    & \qquad \alpha_x \geq 0 \ \ \forall x \in \mathcal X \nonumber \\
    & \qquad c_i^* = \min_{f \in \mathcal F} \underbrace{\mathcal H(W_A\mid X\text{=}x_i)+\alpha_{x_i}\mathcal H(A\mid W_A,X\text{=}x_i)}_{\text{instance level objective}} \ \ \forall x_i \in \mathcal X.\nonumber \\
    & \qquad s^* = \min_{f \in \mathcal F} \eqref{eqn:update_1_constraint} \nonumber,
        \end{align}
    \end{subequations}
    where $\lambda_1$ and $\lambda_2 = \lambda_1\lambda$ are tradeoff parameters, and we have rewritten $\max_f \mathcal H(W_A)$ as $\min_f -\mathcal H(W_A)$. The larger $\lambda_1,\lambda_2$ are, the closer the solutions will be to the optima of the instance level objective and system level simplicity. Empirically, we found that setting both values to $100$ were reasonable.
    \item Finally, we optimize the unconstrained formulation \Cref{eqn:unconstrained_global}, taking the best-of-$N=50$ random seeds, with respect to $f$ to obtain $\mathbf P^*$, the joint distributions over $A$ and $W_A$ for each word in the lexicon. On $\mathbf P^*$, we verify that \Cref{thm:num_feats,thm:inheritance} hold, that is, for some $x \in \mathcal X$, $W_A$ is informative about $A$, and that the number of grammatical values is constrained by the number of semantic attributes. We perform steps 1-5 for various settings of $k$ and $\beta$, keeping $\alpha_x$ and $\lambda_1=\lambda_2=100$ constant, see \Cref{tab:system_level} for a summary of hyperparameter settings.
\end{enumerate}

The choice of $\beta$, which we test in $\{0, 1.0, \to \infty\}$, corresponds to ablations of certain terms in \Cref{eq:objective_app}. In particular,
\begin{enumerate}
    \item $\beta=0$ removes the system-level consistency term.
    \item $\beta \to \infty$ removes the system-level size term.
\end{enumerate}
We want to see, for all tested settings of $k$ and $\beta$, that \Cref{thm:inheritance,thm:num_feats} hold, but also how $\beta$ affects the shape of the system-level distribution.

\subsection*{Results: Instance-level} 
We empirically demonstrate all theoretical results at the instance level: \Cref{lemma:num_feats,lemma:entropy_feat_ref}, that $\alpha_x$ partitions semantic encoding into full specification, underspecification, and neutralization (Equation 6). 

\Cref{fig:gender} shows the results of simulations for $A=$gender, and \Cref{fig:number} for numerosity. Each point in the plots corresponds to the converged $p_{A,W_A\mid X=x}$ for a value of $\alpha_x$ and one random seed. 

In the left plots (blue) of both \Cref{fig:gender,fig:number}, as $\alpha_x$ increases from $0$ to $\infty$, the tradeoff between surprisal (x-axis) and memory (y-axis) moves from a neutralization strategy, seen by memory $\mathcal H(W_A\mid X=x)=0$, or all probability mass on one grammatical value (bottom right corner) to a full specification strategy, seen by surprisal $\mathcal H(A\mid W_A,X=x)=0$ (upper left corner). Values of $\alpha_x$ in the middle correspond to underspecification, where on average, neither extreme of 0 memory not 0 surprisal is achieved.

The center plots (red) illustrate \Cref{lemma:num_feats}, which states that, for a single referent, the number of grammatical values $|W_{A\mid X=x}|$ is capped by the number of semantic categories $|A_{X=x}|$. The graphs show $|W_{A\mid X=x}|$ against $\alpha_x$ increasing from 0 towards $\infty$. The maximum $|W_{A\mid X=x}|$ for each $\alpha_x$ across random seeds are plotted as triangles, and the theoretical upper-bound $|A_{X=x}|$ as a dashed horizontal line. None of the triangles vertically exceed the dashed line, indicating that $|W_{A\mid X=x}| \leq |A_{X=x}|$, with equality at $\alpha_x \to \infty$. On average (round points), $|W_{A\mid X=x}|$ decreases as $\alpha_x$ shrinks to 0. This demonstrates a \emph{compression} of semantic categories into grammatical values as memory requirements are emphasized over decodability.

The right plots (green) demonstrate \Cref{lemma:entropy_feat_ref}, which states that, for a single referent, the \emph{entropy} of grammatical values $\mathcal H(W_{A}\mid X=x)$ is upper-bounded by the entropy of semantic categories $\mathcal H(A \mid X=x)$. Analogous to the center plots, $\mathcal H(W_{A}\mid X=x)$ is shown against $\alpha_x$ from 0 towards $\infty$. Similarly, $\mathcal H(W_{A}\mid X=x) \leq \mathcal H(A\mid X=x)$, with equality at $\alpha_x \to \infty$. On average (round points), $\mathcal H(W_{A}\mid X=x)$ decreases as $\alpha_x$ shrinks to 0, again illustrating \emph{compression} of semantic categories into grammatical values. 

\paragraph{Influence of $\alpha_x$} A granular look at the effect of $\alpha_x$ on the grammatical value distribution for a single referent, on $A=$gender, is shown in \Cref{fig:example_dists_x}. This panel illustrates that semantic encoding strategies are regulated by $\alpha_x$ (see Eq.~3, main), where from left to right, we have neutralization ($\alpha_x=0$), underspecification ($\alpha_x=0.65$), and full specification ($\alpha_x\to \infty)$. Neutralization (figure left) can be seen by all semantic attributes encoded into the same value; underspecification (middle) by the three semantic attributes being compressed to two values; and full specification (right) by the three attributes mapping each to its own value. In particular, we see that the underspecification example mirrors grammatical gender distributions in animate nouns of Romance languages, Slavic languages, and German-- here, we see that a grammatical feminine (top right) refers unambiguously to \emph{+female}, while the grammatical masculine (bottom right) maps back to \emph{+female}, \emph{+male}, and $\varnothing$.

\begin{figure}[htbp]
    \centering
    \begin{subcaption}
        \centering
        \includegraphics[width=\linewidth]{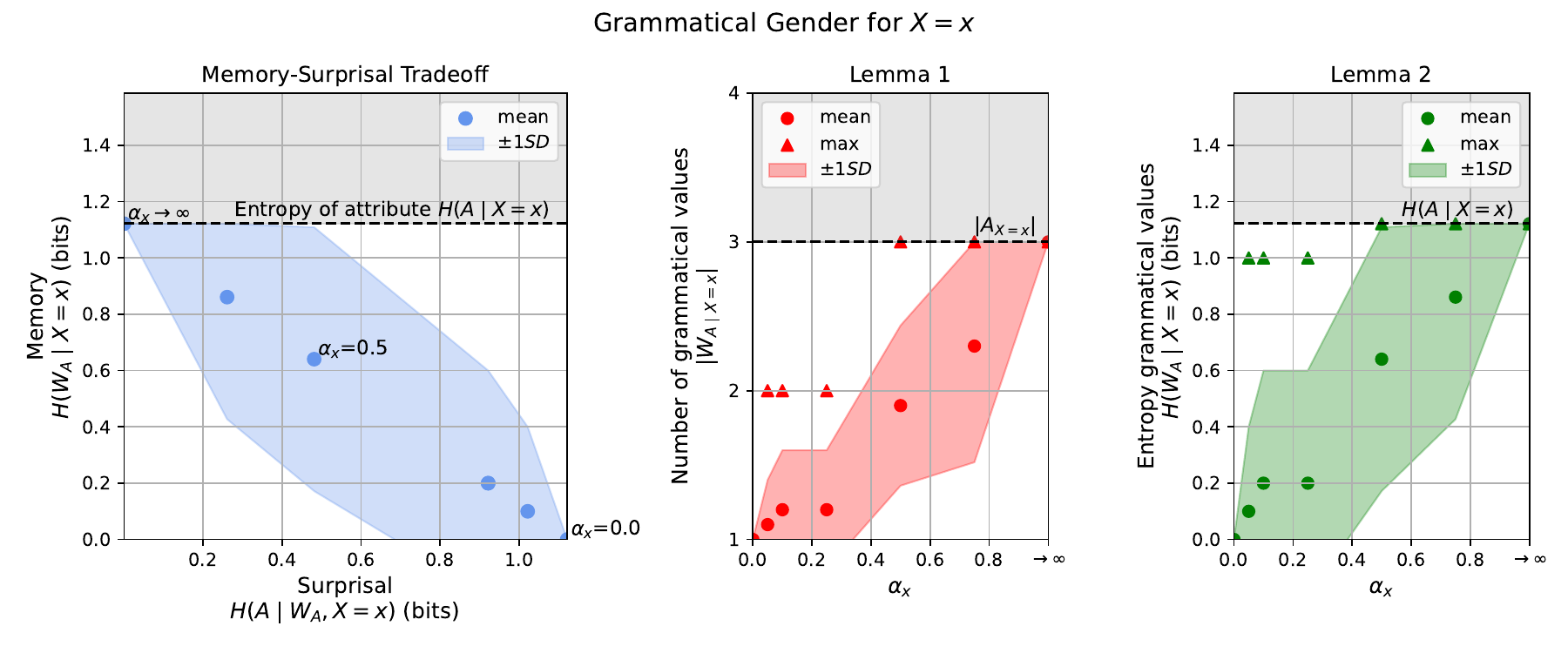}
        \label{fig:gender}
    \end{subcaption}
    
    \vspace{1em} 

    \begin{subcaption}
        \centering
        \includegraphics[width=\linewidth]{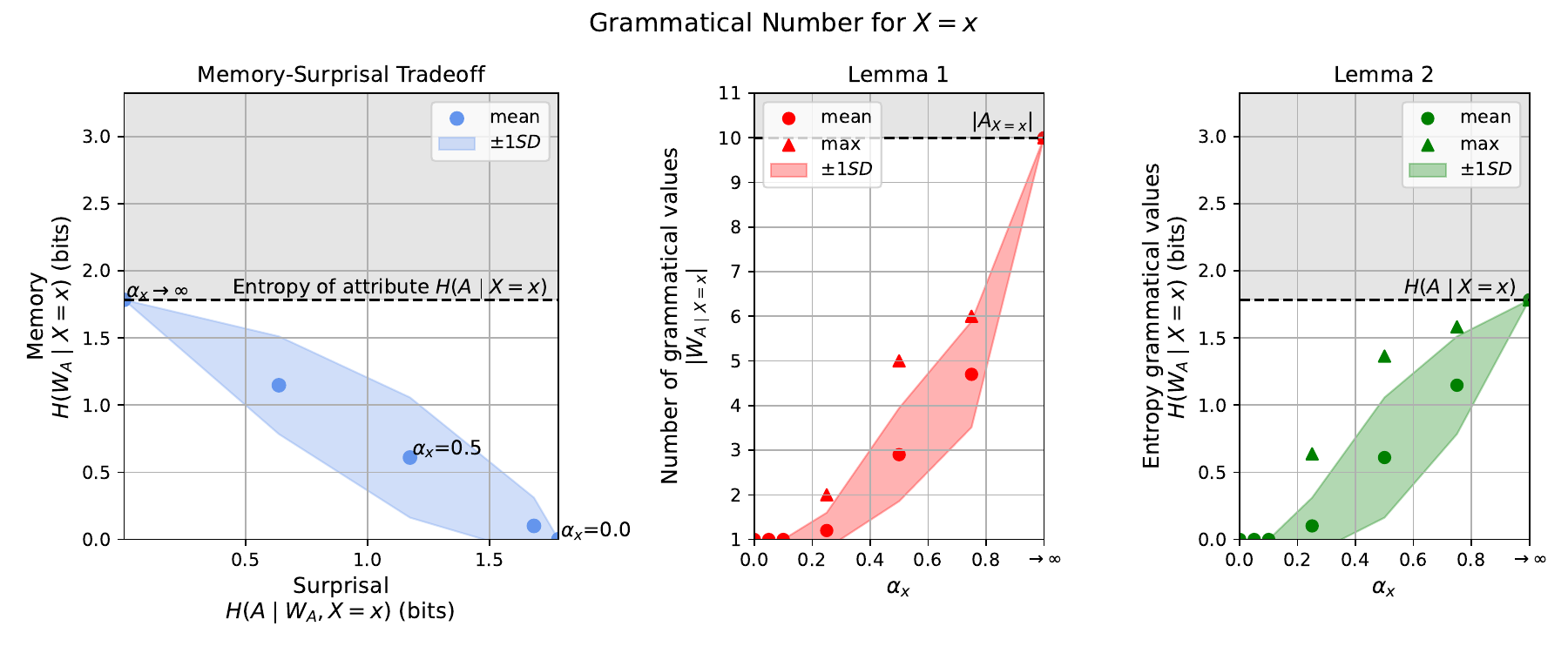}
        \label{fig:number}
    \end{subcaption}

    \caption{\textbf{Instance-level simulations} for hypothetical distributions of \textbf{(a)} gender/biological sex and \textbf{(b)} numerosity for a single referent $x$.}
    \label{fig:sims}
\end{figure}

\begin{figure}
    \centering
    \includegraphics[width=0.8\linewidth]{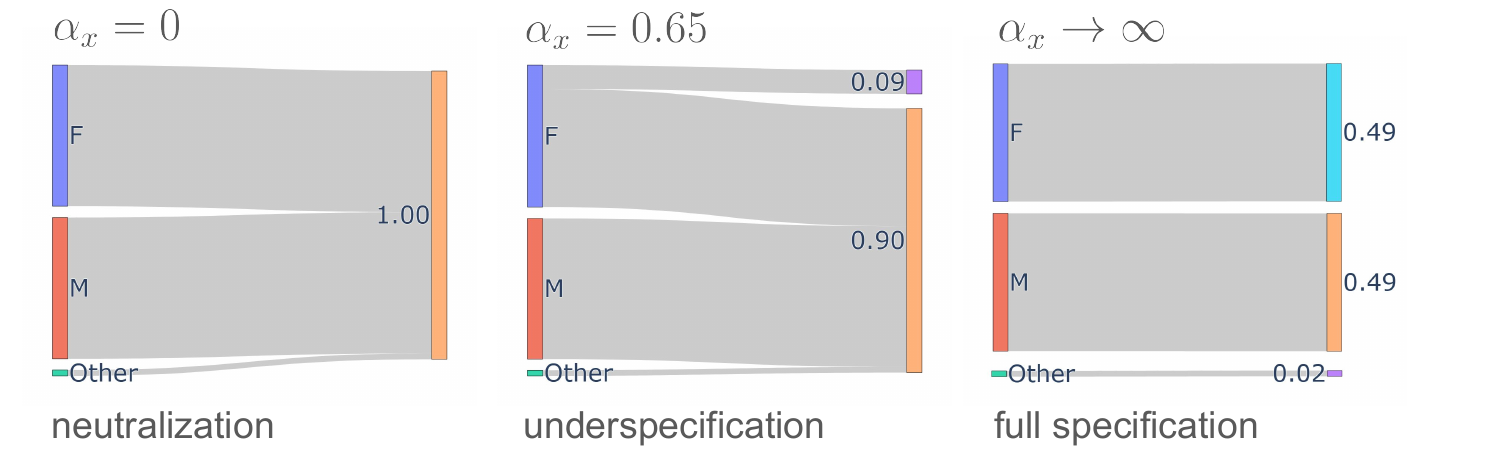}
    \caption{\textbf{Instance-level simulations: semantic encoding strategies}. For various values of $\alpha_x$ and $A=$gender, we show an example of semantic encoding strategies in Box 2 (main text). Here, $\alpha_x = 0$ confers \emph{neutralization}, $\alpha_x=0.65$ \emph{underspecification} (seen by F, M, and other all compressed to the same value), and $\alpha_x\to \infty$ full specification.}
    \label{fig:example_dists_x}
\end{figure}

%

\subsection*{Results: System-level}
Now, we demonstrate, for both number and gender, that \Cref{thm:num_feats,thm:inheritance} hold for a toy lexicon of $k$ words. Our goal is to show, no matter the setting tested, that the number of features in the overall system is bounded above by the number of semantic attributes. As explained in the above section \emph{Marginal distributions}, we set $|A|=3$ for gender (Male, Female, Other) and $|A|=10$ for numerosity. 

We first confirm \Cref{thm:inheritance}. In all tested settings, for gender and number, all $\beta$, and number of words $k \in \{2,4,8,16,32\}$, there was at least one word in the optimized lexicon for which $\mathcal H(A\mid W_A, X=x) < \mathcal H(A\mid X=x)$. \Cref{thm:num_feats} is illustrated in \Cref{fig:thm2}. For gender (left) and number (right), the distributions of $|W_A|$ (y-axis) are summarized for various $\beta$ (different colors) and for increasing number of words in the system (x-axis). The theoretical maximum predicted by \Cref{thm:num_feats}, which for gender is $3$ semantic categories and for number is $10$ semantic categories, is shown as a dashed line. In all cases, the number of grammatical values $|W_A|$ falls below the theoretical maximum, confirming \Cref{thm:num_feats}.

\paragraph{Influence of $\beta$} In \Cref{eq:objective_app}, $\beta \geq 0$ as a tradeoff between size and consistency. While \Cref{thm:num_feats} holds for any $\beta \geq 0$, that is, any tradeoff between size and consistency will produce $|W_A| \leq |A|$, $\beta$ does affect the complexity of the overall grammatical system. \Cref{fig:beta} shows an ablation for gender, where $k=4$ words, $\beta=0$ (no consistency term, size only), and $\beta\to\infty$ (no size term, consistency only). In both cases, $|W_A| \leq |A| = 3$; however, removing the consistency term (left) vastly increases the complexity of the grammatical system, while removing the size term preserves a one-to-one mapping (right) from semantic attribute to grammatical feature. Neither extreme is completely realistic, as real-world grammatical systems are neither extremely complicated nor one-to-one (right); this confirms the need for a tradeoff parameter that interpolates these two modes.

\begin{table}[]
    \centering
        \caption{System-level simulations, hyperparameter settings.}
    \begin{tabular}{c|c|c}
        Name & Description & Values \\
        $\lambda_1$ & simplicity optimality & $10$ \\
        $\lambda$ & instance-level optimality & $10$ \\
        $k$ & number of words & $\{2, 4, 8, 16, 32\}$ \\
        $\beta$ & size / consistency tradeoff & $\{0, 1.0, \to \infty\}$ \\
        $\alpha_x$ & memory / surprisal tradeoff & $k$ evenly-spaced values in $[0,5]$ incl.~boundary
    \end{tabular}
    \label{tab:system_level}
\end{table}

\begin{figure}
    \centering
    \includegraphics[width=0.75\linewidth]{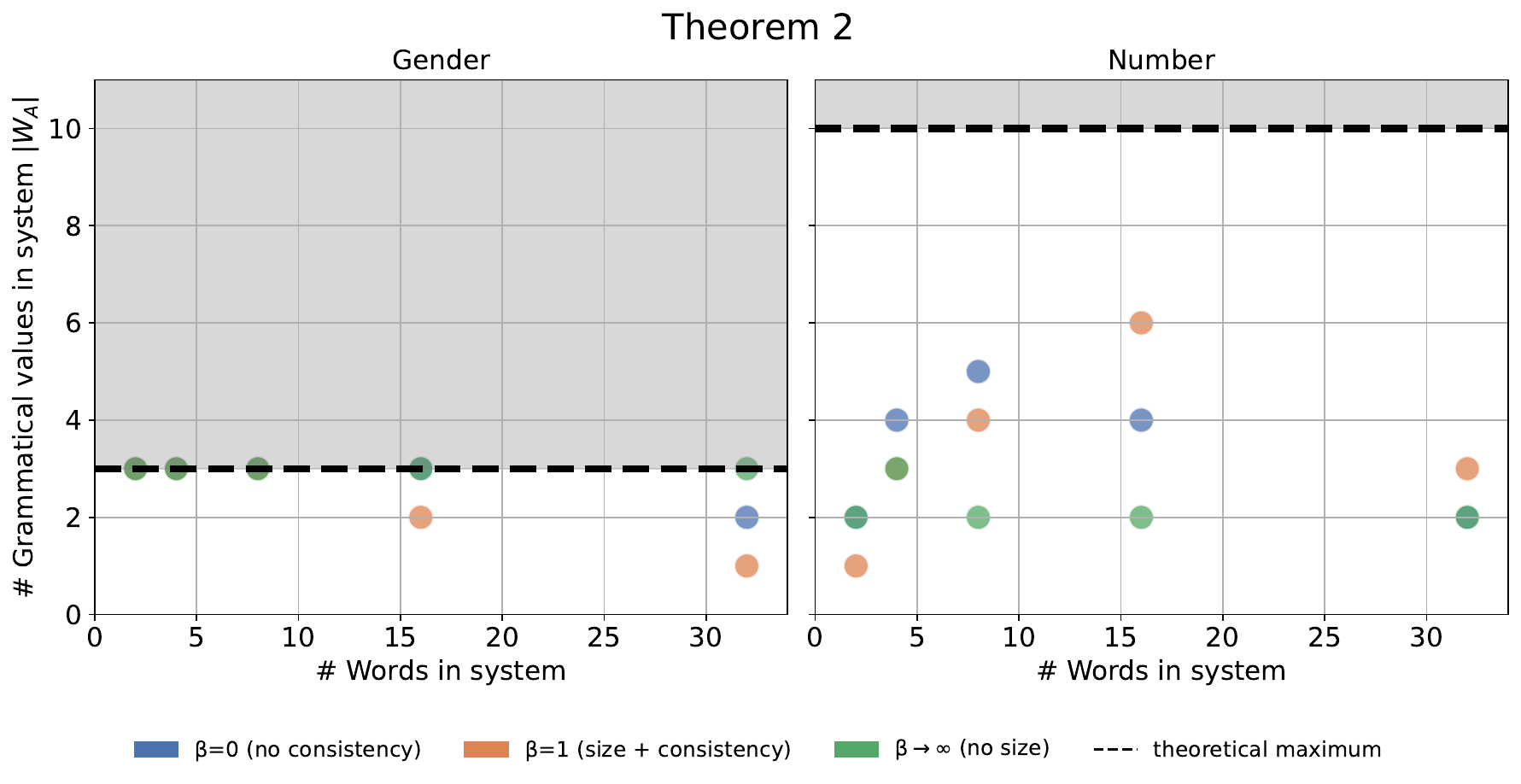}
    \caption{\textbf{System level, number of grammatical values}. We show the number of grammatical values in the system $|W_A|$ (y-axis) for both gender (left) and number (right), against the number of words in the system (x-axis). Here, each point is one optimized grammatical system, where the color indicates the value of $\beta$ (for other hyperparameters, see \Cref{tab:system_level}. The theoretical maximum predicted by \Cref{thm:num_feats}, which is equal to the number of semantic attributes, is plotted as a dark dashed line. Notably, for all of the tested settings, the number of grammatical values is bounded by the number of semantic categories, seen by round points lying under the dashed line.}
    \label{fig:thm2}
\end{figure}

\begin{figure}
    \centering
    \includegraphics[width=0.45\linewidth]{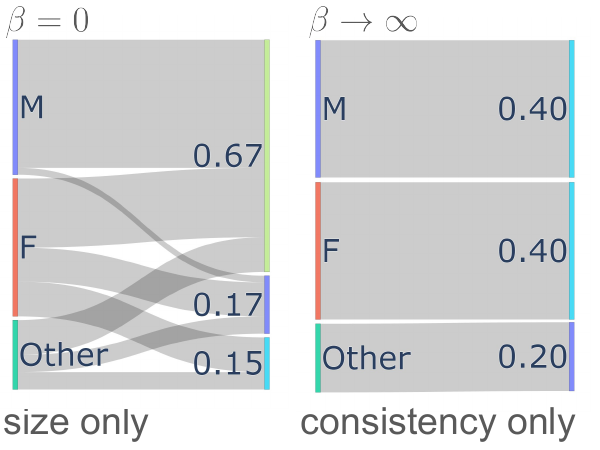}
    \caption{\textbf{Ablations of size, consistency in the system level simplicity term.} Setting $\beta=0$, that is, removing the consistency term (left) vastly increases the complexity of the grammatical system, while setting $\beta \to \infty$, that is, removing the size term, enforces a one-to-one mapping (right) from semantic attribute to grammatical feature.}
    \label{fig:beta}
\end{figure}

\newpage 
\section*{Extended Empirical Results}
\Cref{tab:anim_entropies} shows the grammatical distributions for the animate subset of nouns for all languages. 

\begin{table*}[h]
\centering
\caption{Animate subset. Number of possible values, entropy $\mathcal H$ of inflectional values in bits, and the KL-divergence $\mathcal D_{KL}$ from the optimal uniform distribution Max $\mathcal H$ are reported (left to right) for \textbf{gender} and \textbf{number} features, as well as for the grammatical system restricted to \emph{animate nouns}, for each language. From top to bottom, languages are grouped into Romance, Slavic, Semitic, and Germanic language families.}
\begin{tabular}{@{}lccccccccccccc@{}}
\toprule
& \multicolumn{4}{c}{\textbf{Gender}} & \multicolumn{4}{c}{\textbf{Number}} & \multicolumn{4}{c}{\textbf{Overall}} \\ \cmidrule(lr){2-5} \cmidrule(lr){6-9} \cmidrule(lr){10-13}
Language & \# Values & Max $\mathcal H$ & $\mathcal H$ & $\mathcal D_{KL}$ & \# Values & Max $\mathcal H$ & $\mathcal H$ & $\mathcal D_{KL}$ & \# Values & Max $\mathcal H$ & $\mathcal H$ & $\mathcal D_{KL}$ \\ \midrule
Catalan & $2$ & $1.0$ & $0.74$  & $0.26$ & $2$ & $1.0$ & $0.92$ & $0.08$ & $4$ & $2.0$ & $1.64$ & $0.36^\dagger$  \\
French & $2$ & $1.0$ & $0.66$ & $0.34$ & $2$ & $1.0$ & $0.86$ & $0.14$ & $4$ & $2.0$ & $1.49$ & $0.51$ \\
Italian & $2$ & $1.0$ & $0.61$ & $0.39$ & $2$ & $1.0$ & $0.88$ & $0.12$ & $4$ & $2.0$ & $1.47$ & $0.53$ \\ 
Spanish & $2$ & $1.0$ & $0.81$ & $0.19$ & $2$ & $1.0$ & $0.94$ & $0.06$ & $4$ & $2.0$ & $1.73$ & $0.27^\ddagger$ \\ 
\midrule
Polish & $2$ & $1.58$ & $0.61$ & $0.39$ & $3$ & $1.58$ & $0.93$ & $0.65$ & $6$ & $2.58$ & $1.51$ & $1.07^\ddagger$ \\
Slovene & $2$ & $1.58$ & $0.63$ & $0.37$ & $2$ & $1.0$ & $0.91$ & $0.67$ & $6$ & $2.58$ & $1.50$ & $1.08\ddagger$ \\
\midrule
Arabic & $2$ & $1.0$ & $0.47$ & $0.53$ & $3$ & $1.58$ & $1.0$ & $0.58$ & $6$ & $2.58$ & $1.91$ & $0.68^\ddagger$ \\ 
Hebrew & $2$ & $1.0$ & $0.68$ & $0.32$ & $3$ & $1.58$ & $0.99$ & $0.59$ & $6$ & $2.58$ & $1.65$ & $0.93$ \\
\midrule
German & $2$ & $1.58$ & $0.68$ & $0.32$ & $2$ & $1.0$ & $0.89$ & $0.11$ & $4$ & $2.0$ & $1.51$ & $0.48$  \\
English & $\o$ & -- & -- & -- & $2$ & $1.0$ & $0.93$ & $0.07$ & $2$ & $1.0$ & $0.93$ & $0.07$ \\
Dutch & $2$ & $1.0$ & $0.63$ & $0.37$ & $2$ & $1.0$ & $0.90$ & $0.10$ & $3$ & $1.58$ & $1.50$ & $1.08$  \\ 
Swedish & $2$ & $1.0$ & $0.99$ & $0.01^\dagger$ & $2$ & $1.0$ & $0.33$ & $0.67$ & $4$ & $2.0$ & $1.32$ & $0.67$ \\
\bottomrule
\end{tabular}
\label{tab:anim_entropies}
\end{table*}


\end{document}